%% file: AISTATS_2018_structure.tex
\documentclass[twoside]{article}
\usepackage[accepted]{aistats2018}
\input{preamble}
\begin{document}

\runningtitle{Convex relaxations of combinatorial penalties.}


\twocolumn[

\aistatstitle{Combinatorial Penalties: \\ Which structures are preserved by convex relaxations?}

%\aistatsauthor{ Author 1 \And Author 2 \And  Author 3 }
%
%\aistatsaddress{ Institution 1 \And  Institution 2 \And Institution 3 }

\aistatsauthor{ Marwa El Halabi \And Francis Bach \And  Volkan Cevher }

\aistatsaddress{LIONS, EPFL \And   INRIA - ENS - PSL Research University \And LIONS, EPFL} 
]

\begin{abstract}
We consider the \emph{homogeneous} and the \emph{non-homogeneous} convex relaxations for combinatorial penalty functions defined on support sets. 
%Specifically, we consider parameter models penalized by the sum of an $\ell_p$-norm along with set functions that encode prior knowledge on the support of the unknown vector. 
Our study identifies key differences in the tightness of the resulting relaxations through the notion of the \emph{lower combinatorial envelope} of a set-function along with new necessary conditions for support identification. We then propose a general adaptive estimator for convex monotone regularizers, and derive new sufficient conditions for support recovery in the  asymptotic setting. %We also characterize necessary conditions under which the support of the unknown parameter vector can be correctly identified.  
\end{abstract}

\section{Introduction}

Over the last years, \emph{sparsity} has been a key model in machine learning, signal processing, and statistics. While sparsity modelling is powerful, 
%, where in many applications, one aims at identifying a model of small complexity, well-approximated by a sparse set of coefficients.
\emph{structured sparsity} models further exploit domain knowledge by characterizing the interdependency between the non-zero coefficients of an unknown parameter vector $w$. 
%Domain knowledge may dictate that some non-zero patterns should be favored over others in certain applications; for example: %, non-zero patterns
For example, in certain applications domain knowledge may dictate that we should favor non-zero patterns corresponding to:
%Domain knowledge may dictate, for example, that non-zero patterns corresponding
 unions of groups  \cite{obozinski2011group} in cancer prognosis from gene expression data; or complements of union of groups  \cite{jacob2009group} in neuroimaging and background substraction, or rooted connected trees \cite{jenatton2011proximal, zhao2006grouped} in natural image processing. 
 %, should be favored.
Incorporating such key prior information beyond just sparsity leads to significant improvements in estimation performance,  noise robustness, interpretability and sample complexity \cite{baraniuk2010model}.

Structured sparsity models are naturally encoded by combinatorial functions. However, direct combinatorial treatments often lead to intractable learning problems. Hence, we often use either non-convex greedy methods  or continuous convex relaxations, where the combinatorial penalty is replaced by a tractable convex surrogate; \textit{cf.},  \cite{baraniuk2010model, huang2011learning,bach2011learning}. 

In this paper, we adopt the convex approach because it benefits from a mature set of efficient numerical algorithms as well as strong analysis tools that rely on convex geometry in order to establish statistical efficiency. Convex formulations are also robust to model mis-specifications. Moreover, there is a rich set of 
%However, one difficult question arise within the convex perspective that has gone unresolved so far: Given a desired combinatorial structure, how can we faithfully encode it via a convex function?
%Indeed, structured sparse parameters frequently appear in machine learning, signal processing, and statistics.
%In this setting, several 
convex penalties with structure-inducing properties already studied in the literature
% a priori knowledge on the structure of the support (set of non-zero coefficients) 
\cite{yuan2006model, jacob2009group, jenatton2011structured, jenatton2011proximal, zhao2006grouped, obozinski2011group}. %chandrasekaran2012convex, bach2010structured, obozinski2012convex, halabi2015totally. 
For an overview, we refer the reader to \cite{bach2011learning} and references therein.

%e.g., group Lasso \cite{yuan2006model}, overlapping group Lasso \cite{jacob2009group, jenatton2011structured}, hierarchical group Lasso \cite{jenatton2011proximal, zhao2006grouped}, exclusive Lasso \cite{zhou2010exclusive}, latent group Lasso \cite{obozinski2011group}.
%Other works introduced more general formulations, based on submodular functions \cite{bach2010structured}, atomic norms \cite{chandrasekaran2012convex}, totally unimodular constraints \cite{halabi2015totally}, graph models \cite{hegde2015nearly}, or general $\ell_p$-regularized set functions  \cite{obozinski2012convex}. 
%Non-convex approaches were also proposed in \cite{baraniuk2010model, huang2011learning}.
%For an overview, we refer the reader to \cite{bach2011learning} and references within.
%several proposed stucture inducing penalties are proposed in the literatture (GL, OGL, LGL, HGL, OWL, k support norm..)
%more general systematic approaches: subdmoular functions Bach, atomic norms Venkat, graph structures Piotr, TU structures me.

For choosing a convex relaxation, a systematic approach, already adopted in \cite{bach2010structured, chandrasekaran2012convex, obozinski2012convex, halabi2015totally}, considers the \emph{tightest} convex relaxation of combinatorial penalties expressing the desired structure.  For instance, \cite{bach2010structured} shows that computing the tightest convex relaxation over the unit $\ell_\infty$-ball is tractable for the ensemble of \emph{monotone submodular functions}. 
Similarly, the authors in \cite{halabi2015totally} demonstrates the tractability of such relaxation for combinatorial penalties that can be described via \emph{totally unimodular} constraints.

A different principled approach in convex relaxations is proposed by \cite{obozinski2012convex}, where the authors considered the tightest \emph{homogeneous} convex relaxation of general set functions regularized by an $\ell_p$-norm. The authors show, for instance, the resulting norm takes the form of a generalized latent group Lasso norm \cite{obozinski2011group}. The homogeneity  imposed in \cite{obozinski2012convex}  naturally ensures the invariance of the regularizer to rescaling of the data. However, such requirement may cost a loss of structure as was observed in an example in \cite{halabi2015totally}. %As literature only investigated special cases so far (e.g., for norms associated with submodular functions \cite{bach2010structured}, or for the latent group Lasso norm \cite{obozinski2011group}), 
This observation begs the question: % addressed in this paper: %could be seen

\begin{center}
\begin{minipage}{.75\linewidth}
When do the \emph{homogeneous} and \emph{non-homogeneous} convex relaxations differ and which structures can be encoded by each? 
\end{minipage}
\end{center}

In order to answer this question, we rigorously identify which combinatorial structures are preserved by the non-homogeneous relaxation in a manner similar to  \cite{obozinski2012convex} for the homogeneous one. We
further study the statistical properties of both relaxations. In particular, we consider the problem of support recovery in the context of regularized learning problems by these relaxed convex penalties, which was only investigated so far
in special cases, e.g., for norms associated with submodular functions \cite{bach2010structured}, or for the latent group Lasso norm \cite{obozinski2011group}.

To this end, this paper makes the following contributions:
%In particular, this paper makes the following contributions:
\begin{itemize} \setlength\itemsep{0.2em}
\item We derive formulations of the non-homogeneous tightest convex relaxation of general $\ell_p$-regularized combinatorial penalties (Section \ref{sect:ConvRels}). We show that any \emph{monotone} set function is preserved by such relaxation, while the homogeneous relaxation only preserves a smaller subset of set-functions (Section \ref{sect:LCE}). 
%This is characterized through the notion of \emph{lower combinatorial envelope} (Section \ref{sect:LCE}).
\item We identify necessary conditions for support recovery in learning problems regularized by general convex penalties (Section \ref{sect:NecCond}).
%non-zero patterns to be allowed as solutions to learning problems regularized by convex monotone penalties (Section \ref{sect:NecCond}), and in particular, for regularizers that correspond to convex relaxations of combinatorial functions (Section~\ref{Sect:Disct}).
\item We propose an adaptive weight estimation scheme %based on majorization-minimization 
and provide sufficient conditions for support recovery under the asymptotic regime (Section \ref{sect:SuffCond}). This scheme does not require any irrepresentability condition and is applicable to general monotone convex regularizers.
%These results are applicable  to general monotone convex regularizers and holds even under the correlated design setting. %in linear regression models, 
%\item We translate the sufficient support recovery conditions on the regularizers that correspond to convex relaxations to sufficient conditions on the associated combinatorial functions (Section~\ref{Sect:Disct}).
\item We identify sufficient conditions with respect to combinatorial penalties which ensure that the sufficient support recovery conditions hold with respect to the associated convex relaxations (Section~\ref{Sect:Disct}).
\item  We illustrate numerically the effect on support recovery of the choice of the relaxation %over an example where they differ. We also show that 
as well as the adaptive weights scheme (Section~ \ref{sect:Simul}).
%outperforms non-adaptive ones under both relaxations 
\end{itemize}
In the sequel, we defer all proofs to the Appendix.

%As conclusion say this: 
%Our results, though of theoretical nature, imply important considerations in guiding the design of convex structure-inducing penalties. 
%	•	If the structure prior is not monotone, then it would be useless to consider convex relaxations (homogeneous or not), since they will loose the structure 
%	•	If the structure prior is monotone but not submodular, it is better to use the non-homogeneous relaxation.
%	Adaptive weights allows for support recovery without requiring any irrepresentability condition, even in the correlated measurement matrix setup.
%Given a combinatorial function, the definitions in section 4.1 and results in section 4.2, allows one to check which supports to expect to recover, without the need to compute the corresponding convex relaxation. 

\paragraph{Notation.}

Let $V = \{1,\dots,d\}$ be the ground set and $2^V = \{A | A \subseteq V\}$ be its power-set.
Given $w \in \R^d$ and a set $J \subseteq V$, $w_J$ denotes the vector in $\R^d$ s.t., $[w_J]_i = w_i,  i \in J$ and $[w_J]_i = 0,  i \not \in J$. $Q_{JJ}$ is defined similarly for a matrix $Q \in \R^{d \times d}$.
% $Q_{JJ}$ denote the corresponding subvector and submatrix of $w$ and $Q$. 
Accordingly, we let $\1_J$ be the indicator vector of the set $J$. %and accordingly $\1_i$ is the $i$-th basis vector. 
We drop the subscript for $J = V$, so that $\1_V = \1$ denotes the vector of all ones. The notation $J^c$ denotes the set  complement of $J$ with respect to $V$.

The operations $|w|, w \geq w'$ and $w \circ v$ are applied element-wise. 
%Similarly, the comparison $w \geq w'$ and the product $w \circ v$ are also applied element wise. 
For $p > 0$, the $\ell_p$-(quasi) norm is given by $\| w \|_p = (\sum_{i=1}^d |w_i|^{p})^{1/p}$, and $\| w \|_\infty = \max_i |w_i|$.
For $p \in [1,\infty]$, we define the conjugate $q \in [1,\infty]$ via $\frac{1}{p} + \frac{1}{q} = 1$.

We call the set of non-zero elements of a vector $w$ the support, denoted by $\supp(w) = \{i : w_i \not = 0\} $. 
We use the notation from submodular analysis, $w(A) = \sum_{i \in A} w_i$.
We write $\overline{\R}_+$ for $\R_+ \cup \{+\infty\}$.  For a function $f : \R^d \rightarrow \overline{\R} =  \R  \cup \{+\infty\}$, we will denote by $f^\ast$ its Fenchel-Legendre conjugate.
We will denote by $\iota_S(w)$ the indicator function of the set $S$, taking value $0$ on the set $S$ and $+\infty$ outside it. %Finally, we extend $b \rightarrow \frac{a}{b}$ by continuity in zero with $\frac{a}{0} = \infty$ if $a \not = 0$ and $0$ otherwise. I use this in the appendix so will state it there.

\input{ConvRel}
\input{ConvOpt}

\input{DiscreteStable}

\input{Simulations}

\section{Conclusion}

We presented an analysis of homogeneous and non-homogeneous convex relaxations of $\ell_p$-regularized combinatorial penalties. Our results show that structure encoded by submodular priors can be equally well expressed by both relaxations, while the non-homogeneous relaxation is able to express the structure of more general monotone set functions. We also identified necessary and sufficient stability conditions on the supports to be correctly recovered. We proposed an adaptive weight scheme that is guaranteed to recover supports that satisfy the sufficient stability conditions, in the asymptotic setting, even under correlated design matrix.

\subsubsection*{Acknowledgements}
We thank Ya-Ping Hsieh for helpful discussions.
This work was supported in part by the European Commission under ERC Future Proof, SNF 200021-146750,
SNF CRSII2-147633, NCCR Marvel. Francis Bach acknowledges support from the chaire Economie des nouvelles donn\'ees with the data science joint research initiative with the fonds AXA pour la recherche, and the Initiative de Recherche ``Machine Learning for Large-Scale Insurance'' from the Institut Louis Bachelier.

{ 
\bibliographystyle{plain}
\bibliography{biblio}
}

\newpage 

\input{Appendix}

\end{document}

%% file: preamble.tex
\PassOptionsToPackage{hyphens}{url}\usepackage{hyperref}
\usepackage{color}
 \usepackage{tikz} % for plotting graphs
 \usetikzlibrary{arrows,automata}
\usepackage{times,amsmath,amssymb,amsfonts,epsfig,graphicx, mathrsfs}
\usepackage{amsthm}  % defines the 'proof' environment with square symbol
\usepackage{thmtools,thm-restate}
 \usepackage{dsfont}
 \usepackage{enumitem}
 
\newtheorem{definition}{Definition}

\newtheorem{proposition}{Proposition}
\newtheorem{lemma}{Lemma}
\newtheorem{corollary}{Corollary}

\newtheorem{remark}{Remark}
\newtheorem{example}{Example}

\usepackage[T3,T1]{fontenc}
\DeclareSymbolFont{tipa}{T3}{cmr}{m}{n}
\DeclareMathAccent{\invbreve}{\mathalpha}{tipa}{16}

\newcommand{\intr}{\mathrm{int}}
\newcommand{\dom}{\mathrm{dom}}
\newcommand{\sign}{\mathrm{sign}}

\newcommand{\supp}{\mathrm{supp}}

\newcommand{\R}{\mathbb{R}}

\newcommand{\M}{\mathcal{M}}

\newcommand{\G}{\mathcal{G}}

\newcommand{\1}{\mathds{1}}

\newcommand{\colcst}[1]{\substack{#1}}
\DeclareMathOperator*{\argmax}{arg\,max}
\DeclareMathOperator*{\argmin}{arg\,min}

%% file: ConvRel.tex
%!TEX root = AISTATS_2018_structure.tex
\section{Combinatorial penalties and convex relaxations}
%Structured sparsity models are naturally expressed by combinatorial functions.
 We  consider positive-valued set functions of the form $F:2^V \rightarrow \overline{\R}_+$ such that $F(\varnothing) = 0, F(A)>0, \forall A \subseteq V$ to encode structured sparsity models. 
 For generality, we do not assume a priori that $F$ is \emph{monotone} (i.e., $F(A) \leq F(B), \forall A \subseteq B$). However, as we argue in the sequel,  convex relaxations of non-monotone set functions is  hopeless.

The domain of $F$ is defined as $\mathcal{D} := \{A: F(A) < +\infty\}$. We assume that it covers $V$, i.e., $\cup_{A \in \mathcal{D}} A = V$, which is equivalent to assuming that $F$ is finite at singletons if $F$ is monotone.
A finite-valued set function
$F$ is %\emph{monotone} if $F(A) \leq F(B), \forall A \subseteq B$, and 
\emph{submodular} if and only if for all $A\subseteq B$ and $i \in B^c$, $F(B \cup \{i\}) - F(B) \leqslant F(A \cup \{i\}) - F(A)$~(see, e.g., \cite{fujishige2005submodular,bach2011learning}). Unless explicitly stated, we do not assume that $F$ is submodular.

We consider the same model in \cite{obozinski2012convex}, parametrized by $w \in \R^d$,  with general $\ell_p$-regularized combinatorial penalties:% of the form 
$$F_p(w) = \frac{1}{q} F(\supp(w)) + \frac{1}{p} \| w \|_p^p$$ % \footnote{Any positive scalar coefficients can be used but the dual pairs $(\frac{1}{p},\frac{1}{q})$ make for a neater presentation}
for $p \geq 1$, where the set function $F$ controls the structure of the model in terms of allowed/favored non-zero patterns and the $\ell_p$-norm serves to control the magnitude of the coefficients. Allowing $F$ to take infinite values let us
enforce hard constraints. %F can be chosen to be an indicator function over a set of allowed sparsity patterns.
For $p = \infty$, $F_p$ reduces to $F_\infty(w) = F(\supp(w)) + \iota_{\| w \|_\infty \leq 1}(w)$.
Considering the case $p \not = \infty$ is appealing to avoid the clustering artifacts of the values of
the learned vector induced by the $\ell_\infty$-norm.

Since such combinatorial regularizers lead to computationally intractable problems, we seek convex surrogate penalties that capture the encoded structure as much as possible. A natural candidate for a convex surrogate of $F_p$ is then its \emph{convex envelope} (largest convex lower bound), given by the biconjugate (the Fenchel conjugate of the Fenchel conjugate) $F_p^{\ast \ast}$. Two general approaches are proposed in the literature to achieve just this; one requires the surrogate to also be positively homogeneous \cite{obozinski2012convex} and thus considers the convex envelope of the positively homogeneous envelope of $F_p$, given by $F(\supp(w))^{1/q} \| w \|_p$, %there's no constant factor actually
which we denote by $\Omega_p$, the other computes instead the convex envelope of $F_p$ directly \cite{halabi2015totally, bach2010structured}, which we denote by $\Theta_p$. Note that from the definition of convex envelope, it holds that $\Theta_p \geq \Omega_p$.

\subsection{Homogeneous and non-homogeneous convex envelopes} \label{sect:ConvRels}

In \cite{obozinski2012convex}, the homogeneous convex envelope $\Omega_p$ of $F_p$ was shown to correspond to the \emph{latent group Lasso norm} \cite{obozinski2011group} with groups set to all elements of the power set $2^V$. We recall this form of $\Omega_\infty$ in Lemma \ref{lem:HomEnv} as well as a variational form of $\Omega_p$ which highlights the relation between the two. Other variational forms can be found in the Appendix.

\begin{lemma}[\cite{obozinski2012convex}] \label{lem:HomEnv}
The homogeneous convex envelope $\Omega_p$ of $F_p$ is given by
{\small
\begin{align}
\Omega_p(w) 
&=\inf_{\eta \in \R^d_+}  \frac{1}{p} \sum_{j=1}^d \frac{|w_j|^{p}}{\eta_j^{p -1}} + \frac{1}{q} \Omega_\infty(\eta), \label{eq:VarLpLinf}  \\
\Omega_\infty(w) &=  \min_{\alpha \geq 0} \Big\{ \sum_{S \subseteq V} \alpha_S F(S) :  \sum_{S \subseteq V} \alpha_S \1_S \geq |w| \Big\}.
\label{eq:ConvCover}
\end{align}
}
\end{lemma}

The non-homogeneous convex envelope $\Theta_p$ of $F_p$ is only considered thus far in the case where $p = \infty$. 
%of a set function $F$,  over the unit $\ell_\infty$-ball 
\cite{halabi2015totally}  shows that $\Theta_\infty(w) = \inf_{\eta \in [0,1]^d} \{ f(\eta) : \eta \geq |w|\}$ where $f$ is any proper $(\dom(f) \not = \emptyset)$, lower semi-continuous (l.s.c.) convex \emph{extension} of $F$, i.e., $f(\1_A) = F(A), \forall A\subseteq V$ (\textit{cf.}, Lemma 1 in \cite{halabi2015totally}). A natural choice for $f$ is the \emph{convex closure} of $F$, which corresponds to the \emph{tightest} convex extension of $F$ on $[0,1]^d$  (\textit{cf.}, Appendix for a more rigorous treatment).

Lemma \ref{lem:NonHomEnv} below presents this choice, deriving a new form of $\Theta_\infty$ that parallels \eqref{eq:ConvCover}.
 We also derive the non-homogeneous convex envelope $\Theta_p$ of $F_p$ for any $p \geq 1$ and present the variational form relating it to $\Theta_\infty$ in Lemma \ref{lem:NonHomEnv}.
 %The following new lemma generalizes it to . 
 For simplicity, the variational form~\eqref{eq:VarLpLinfNonHomo} presented below holds only for monotone functions $F$; the general form and other variational forms that parallel the ones known for the homogeneous envelope are presented in the Appendix.

\begin{lemma}\label{lem:NonHomEnv}
The non-homogeneous convex envelope $\Theta_p$ of $F_p$, for monotone functions $F$, is given by
{\small
\begin{align}
\Theta_p(w) 
&=\inf_{\eta \in [0,1]^d}  \frac{1}{p} \sum_{j=1}^d \frac{|w_j|^{p}}{\eta_j^{p -1}} + \frac{1}{q} \Theta_\infty(\eta),
\label{eq:VarLpLinfNonHomo}  \\ 
 \Theta_\infty(w)  &= \min_{\alpha \geq 0} \Big\{ \sum_{S \subseteq V} \alpha_S F(S) :  \sum_{S \subseteq V} \alpha_S \1_S \geq |w|, \sum_{S \subseteq V} \alpha_S =1 \Big\}. %, \forall w \in [-1,1]^d.
 \label{eq:ConvCoverNonHomo}
  %no need to restrict w more, if w is such that no s covers it in the domain of f-, then the inf will be infinity, so this is well defined
\end{align}
}
%otherwise need to replace with convex closure
%and where we define
%\begin{align*}
%\psi_j(\kappa_j,w_j)  &:=\begin{cases} \kappa_j^{1/q}|w_j| &\text{ if $|w_j| \leq \kappa_j^{1/p}$}\\
%\frac{1}{p}  |w_j|^{p}+ \frac{1}{q}   \kappa_j  &\text{otherwise.}
%\end{cases} 
%\end{align*} 
\end{lemma}
The infima in \eqref{eq:VarLpLinf} and \eqref{eq:VarLpLinfNonHomo}, for $w \in \dom(\Theta_p)$, can be replaced by a minimization, if we extend $b \rightarrow \frac{a}{b}$ by continuity in zero with $\frac{a}{0} = \infty$ if $a \not = 0$ and $0$ otherwise, as suggested in \cite{jenatton2010structured} and \cite{bach2012optimization}.
Note that, for $p = 1$, both relaxations reduce to $\Omega_1 =\Theta_1 = \| \cdot \|_1$.  Hence, the $\ell_1$-relaxations essentially lose the combinatorial structure encoded in $F$. We thus follow up on the case $p > 1$.

In order to decide when to employ $\Omega_p$ or $\Theta_p$, it is of interest to study the respective properties of these two relaxations and to identify when they coincide. Remark \ref{rmk:SubNonHomoEnv} shows that the
homogeneous and non-homogeneous envelopes are identical, for $p=\infty$, for monotone submodular functions.

\begin{remark}\label{rmk:SubNonHomoEnv}
If $F$ is a monotone submodular function then $\Theta_\infty(w) =  \Omega_\infty(w) = f_L(|w|), \forall w \in [-1,1]^d$, where $f_L$ denotes the Lov\'asz extension of $F$ \cite{lovasz1983submodular}. % i.e., homogeneous and non-homogeneous envelopes are identical, on the unit $\ell_\infty$-ball, for monotone submodular functions.
\end{remark}

%We defer the proof of lemma \ref{lem:NonHomVarForms} and remark \ref{rmk:SubNonHomoEnv} to the appendix.

%\iffalse
\begin{figure*}\label{fig:l0l2fct}
\vspace{-10pt}
\centering
\begin{tabular}{c c c}
%\hspace{-20pt}
\includegraphics[trim=53 200 20 200, clip, scale=.22]{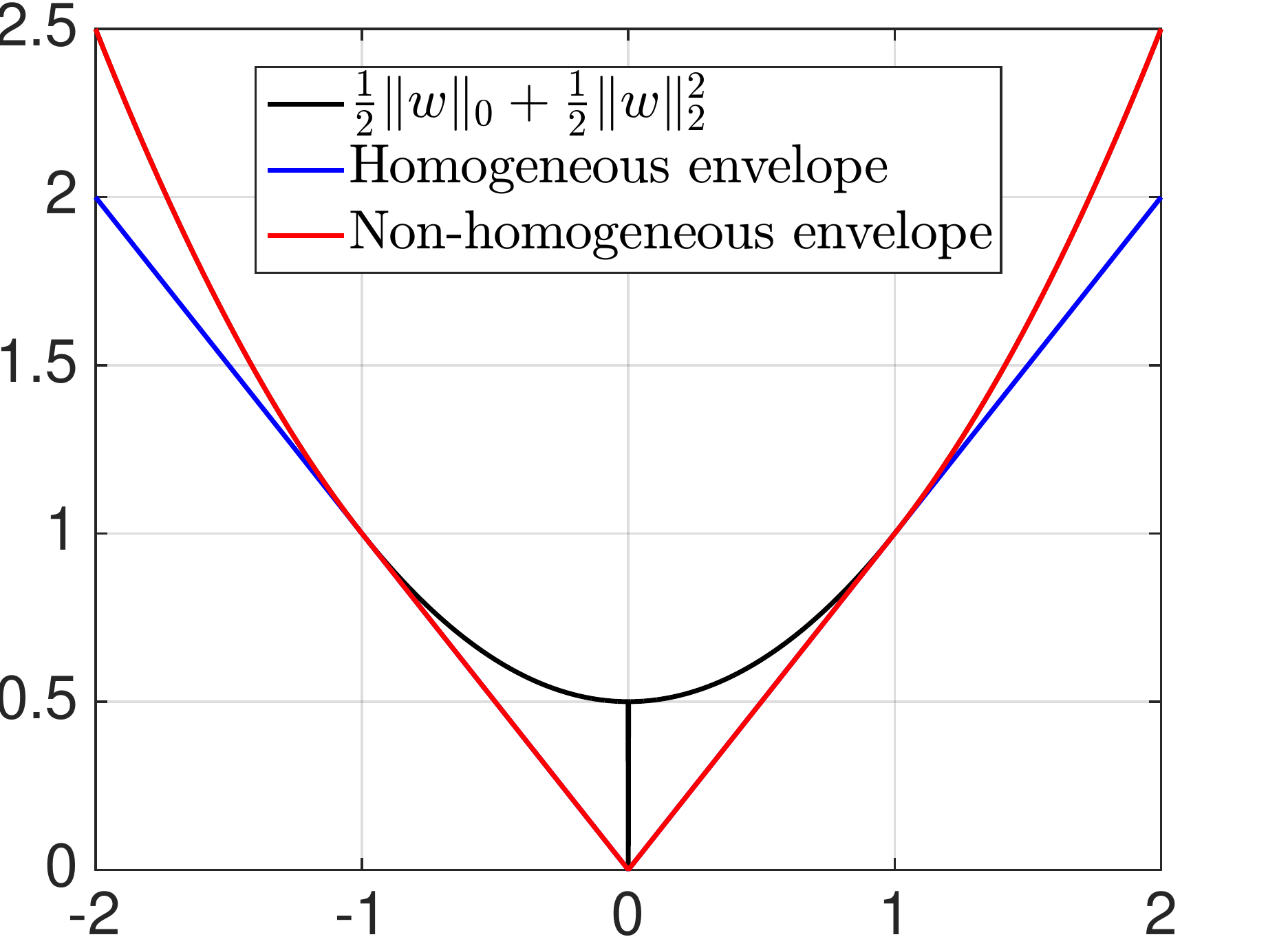} &
\includegraphics[scale=.25]{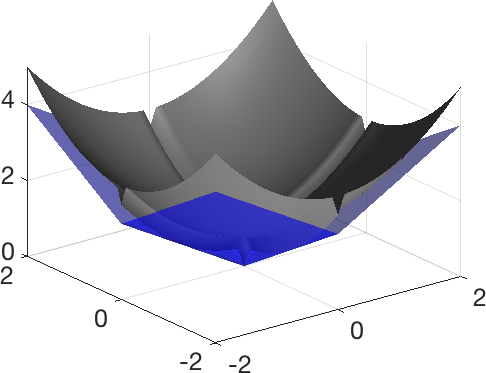} &
\includegraphics[scale = .25]{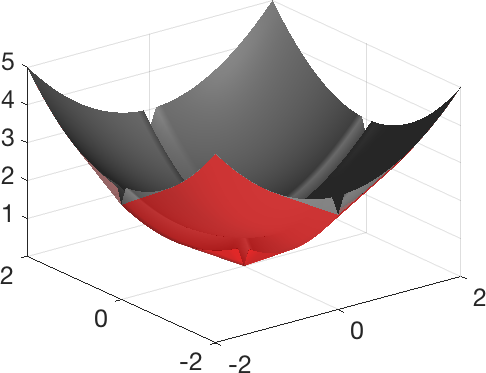} 
\end{tabular}
\caption{$\ell_2$-regularized cardinality example in one dimension (left) and two dimensions (middle: homogeneous, right: non-homogeneous).  }
\end{figure*}
%\fi

The two relaxations do not coincide in general: Note the added constraints $\eta \in [0,1]^d$ in (\ref{eq:VarLpLinfNonHomo}) and the sum constraint on $\alpha$ in (\ref{eq:ConvCoverNonHomo}).
Another clear difference to note is that $\Omega_p$ are norms that belong to the broad family of H-norms \cite{micchelli2013regularizers, bach2012optimization}, as shown in \cite{obozinski2012convex}. %Bach Sec. 1.4.2.
On the other hand, by virtue of being non-homogeneous, $\Theta_p$ are not norms in general. 
We illustrate below two interesting examples where $\Omega_p$ and $\Theta_p$ differ.

\begin{example}[Berhu penalty] Since the cardinality function $F(S) = |S|$ is a monotone submodular function, $\Theta_\infty(w) = \Omega_\infty(w) = \| w \|_1$. However, this is not the case for $p\not = \infty$. In particular, we consider the $\ell_2$-regularized cardinality function $F^{card}_2(w) = \frac{1}{2}\| w \|_0 + \frac{1}{2} \| w \|_2^2$. 
Figure \ref{fig:l0l2fct} shows that the non-homogeneous envelope is \emph{tighter} than the homogeneous one in this case. Indeed, $\Omega^{card}_2$ is simply the $\ell_1$-norm, while $\Theta^{card}_2$  is given by $[\Theta^{card}_2(w)]_i = |w_i|$ if $|w_i| \leq 1$ and $[\Theta^{card}_2(w)]_i = \frac{1}{2} |w_i|^2 + \frac{1}{2}$ otherwise. This penalty, called ``Berhu,'' is introduced in \cite{owen2007robust} to produce a robust ridge regression estimator and is shown to be the convex envelope of  $F^{card}_2$ in~\cite{jojic2011convex}.
\end{example}

This kind of behavior, where the non-homogeneous relaxation $\Theta_p$ acts as an $\ell_1$-norm on the small coefficients and as $\frac{1}{p} \| w\|_p^p$ for large ones, is not limited to the Berhu penalty, but holds for general set functions. However the point where the penalty moves from one mode to the other depends on the structure of $F$ and is different along each coordinate.
This is easier to see via the second variational form of  $\Theta_p$ presented in the Appendix. We further illustrate in the following example.

\begin{example}[Range penalty] Consider the range function defined as $\text{range}(A) = \max(A) - \min(A) + 1$ where $\max(A)$ ($\min(A)$) denotes maximal (minimal) element in $A$. This penalty allow us to favor the selection of interval non-zero patterns on a chain or rectangular patterns on grids. It was shown in \cite{obozinski2012convex} that $\Omega_p(w) = \| w \|_1$ for any $p \geq 1$. On the other hand, $\Theta_p$ has no closed form solution, but is different from $\ell_1$-norm. Figure \ref{fig:rangeBalls} illustrates the balls of different radii of $\Theta_\infty$ and $\Theta_2$. We can see how the penalty morphs from $\ell_1$-norm to $\ell_\infty$ and squared $\ell_2$-norm respectively, with different ``speed'' along each coordinate. Looking carefully for example on the ball $\Theta_2(w) \leq 2$, we can see that the penalty acts as an $\ell_1$-norm along the $(x,z)$-plane and as a squared $\ell_2$-norm along the $(y,z)$-plane.
\end{example}

We highlight other ways in which the two relaxations differ and their implications in the sequel. % the following sections.

In terms of computational efficiency, note that even though
the formulations (\ref{eq:VarLpLinf}) and (\ref{eq:VarLpLinfNonHomo}) are jointly convex in $(w,\eta)$, % , as observed in \cite{obozinski2012convex},
%In general however, 
$\Omega_p$ and $\Theta_p$ can still be intractable to compute and to optimize. 

However, for certain classes of functions, they are tractable. For example, since for monotone submodular functions, $\Omega_\infty = \Theta_\infty$ is the Lov\'asz extension of $F$, as stated in Remark \ref{rmk:SubNonHomoEnv}, then they can be efficiently computed by the greedy algorithm~\cite{bach2011learning}. Moreover, efficient algorithms to compute $\Omega_p$, the associated proximal operator and to solve learning problems regularized with $\Omega_p$ is proposed in \cite{obozinski2012convex}. 
Similarly, if $F$ can be expressed by integer programs over totally unimodular constraints as in \cite{halabi2015totally}, then $\Omega_\infty$, $\Theta_\infty$ and their associated Fenchel-type operators can be computed efficiently by linear programs. Hence, we can use conditional gradient algorithms for numerical solutions. %Luckily, these two tractable classes of combinatorial penalties, include many structures of interest.

%\iffalse
\begin{figure}\label{fig:rangeBalls}
\centering 
\begin{tabular}{c c c}
\hspace{-15pt}
\includegraphics[scale=.25]{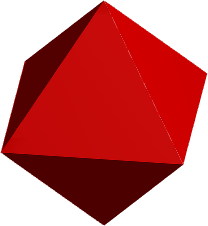} &
\includegraphics[scale=.33]{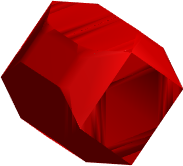} &
\includegraphics[scale=.15]{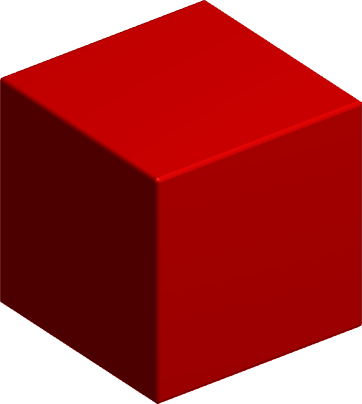} \\
\hspace{-15pt}
\includegraphics[scale=.25]{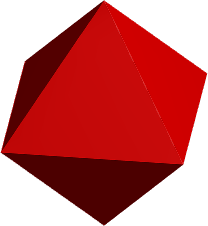} &
\includegraphics[scale=.28]{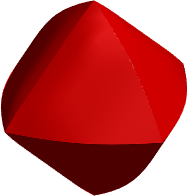} &
\includegraphics[scale=.36]{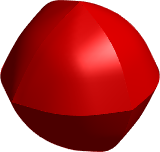} \\
\end{tabular}
%\caption{Balls of radius $r = 1$ (left), $r = 2$ (middle), $r=4$ (right) of the $\ell_\infty$-convex envelope (top) and $\ell_2$-convex envelope (bottom) of the range function.}
\caption{Balls of different radii of the non-homogeneous $\ell_\infty$-convex envelope of the range function (top): $\Theta_\infty(w) \leq 1$ (left), $\Theta_\infty(w) \leq 2$ (middle), $\Theta_\infty(w) \leq 3$ (right)  and of its $\ell_2$-convex envelope (bottom): $\Theta_2(w) \leq 1$ (left), $\Theta_2(w) \leq 2$ (middle), $\Theta_2(w) \leq 4$ (right).}
\vspace{-12pt}
\end{figure}
%\fi
\subsection{Lower combinatorial envelopes} \label{sect:LCE}

In this section, we are interested in analyzing which combinatorial structures are preserved by each relaxation.
%we characterize the tightness of the two convex relaxations. 
To that end, we generalize
the notion of \emph{lower combinatorial envelope} (LCE)  \cite{obozinski2012convex}.
The homogeneous LCE $F_-$  of $F$ is defined as the set function which agrees with the $\ell_\infty$-homogeneous convex relaxation of $F$ at the vertices of the unit hypercube, i.e.,  $F_-(A) = \Omega_\infty(\1_A), \forall A \subseteq V$. 

For the non-homogeneous relaxation, we define the non-homogeneous LCE similarly as $\tilde{F}_-(A) = \Theta_\infty(\1_A)$.
The $\ell_\infty$-relaxation reflects most directly the combinatorial structure of the function $F$. %, hence defining the LCE through it makes sense. 
Indeed, $\ell_p$-relaxations only depend on $F$ through the $\ell_\infty$-relaxation as expressed in the variational forms \eqref{eq:VarLpLinf} and \eqref{eq:VarLpLinfNonHomo}.

We say $\Omega_\infty$ is a tight relaxation of $F$ if $F = F_-$. Similarly, $\Theta_\infty$ is a tight relaxation of $F$ if $\tilde{F}_-  = F$. $\Omega_\infty$ and $\Theta_\infty$ are then \emph{extensions} of $F$ from $\{0,1\}^d$ to $\R^d$; in this sense, the relaxation is tight for all $w$ of the form $w=\1_A$. Moreover, following the
 definition of convex envelope, the relaxation
 $\Omega_\infty$  (resp.~$\Theta_\infty$) is the same for $F$ and $F_-$ (resp.~ $F$ and $\tilde{F}_-$), and hence, the LCE can be interpreted as the combinatorial structure preserved by each convex relaxation.

The homogeneous relaxation  can capture any monotone submodular function \cite{obozinski2012convex}. Since $\Omega_\infty$ is the Lov\'asz extension \cite{bach2010structured} in this case, and hence, $F_-(A) = \Omega_\infty(\1_A) = f_L(\1_A) = F(A)$. Also, since the two $\ell_\infty$-relaxations are identical for this class of functions, their LCEs are also equal, i.e., $\tilde{F}_-(A) = \Theta_\infty(\1_A) = \Omega_\infty(\1_A) = F(A)$.
%and by Remark \ref{rmk:SubNonHomoEnv} $\tilde{F}_-(A) = \Theta_\infty(\1_A) = \Omega_\infty(\1_A) = F(A)$. Hence, both relaxations are tight in the case of monotone submodular functions, and the two LCEs are equal. 

The LCEs, however, are not equal in general. In fact, the non-homogeneous relaxation is tight for a larger class of functions.
%The LCE value $F_-(A)$ can be interpreted, via the variational form \eqref{eq:ConvCover}, as the minimal fractional weighted set-cover $A$, a classical relaxation of the minimal weighted set-cover problem \cite{lovasz1975ratio}, as noted in \cite{obozinski2012convex}. It is in general not equal to $F(A)$.  
In particular, the following proposition shows
 that $\tilde{F}_-$ is equal to the \emph{monotonization} of $F$, that is $\tilde{F}_-(A) = \inf_{S \subseteq V}  \{ F(S) :  A \subseteq S\}$, for all set functions $F$, and is thus equal to the function itself if $F$ is monotone.
 
\begin{restatable}{proposition}{primeLCEnonhomoProp}\label{prop:LCENonHomo}
%\begin{proposition}\label{prop:LCENonHomo}
The non-homogenous lower combinatorial envelope can be written as
{\small
\begin{align*}
\tilde{F}_-(A) &= \Theta_\infty(\1_A) \\ %&=  \inf \{ \sum_{S \subseteq V} \alpha_S F(S) :  \sum_{S \subseteq V} \alpha_S \1_S \geq \1_A, \sum_{S \subseteq V} \alpha_S =1, \alpha_S \geq 0\}  \\
&= \!\! \inf_{\alpha_S \in \{0,1\}} \{ \sum_{S \subseteq V} \alpha_S F(S) :  \sum_{S \subseteq V} \! \alpha_S \1_S \geq \1_A, \sum_{S \subseteq V} \! \alpha_S =1 \} \\[-.1cm]
&=  \inf_{S \subseteq V}  \{ F(S) :  A \subseteq S\}. \\[-.6cm]
\end{align*}
}
%\end{proposition}
\end{restatable}
\begin{proof}
To see why we can restrict $\alpha_S$ to be integral, let $\mathcal{E} = \{S: \alpha_S>0\}$, then $\forall T \subseteq V$ such that $\exists e \in A,  e \not \in T$, 
then $\sum_{\alpha_S >0, S \not = T} \alpha_S = 1$ and hence $\alpha_{T}=0$. Hence $\forall S \in \mathcal{E}$ we have $A \subseteq S$ and
 $ \sum_{\alpha_S > 0} \alpha_S F(S) = \min_{\alpha_S > 0} F(S)$. 
\end{proof}
%The non-homogeneous convex envelope is thus always tight for monotone functions and hence a ``tighter" relaxation than the homogeneous one. 
Proposition \ref{prop:LCENonHomo} argues that the non-homogeneous convex envelope is tight if and only if $F$ is monotone. Two important practical implications follow from this result. 

Given a target model that cannot be expressed by a monotone function, it is impossible to obtain a tight convex relaxation. Non-convex methods can be potentially better. % , and one has to resort instead to non-convex methods. 

On the other hand, if the model can be expressed by a monotone non-submodular set function, the homogeneous function may not be tight, and hence, a non-homogeneous relaxation can be more useful. For instance, \cite{obozinski2012convex} shows that for any set function where $F(\{e\})=1$ for all singletons $e \in V$ and $F(A) \geq |A|, \forall A \subseteq V$, the homogeneous LCE $F_-(A) = |A|$ and accordingly $\Omega_p$ is the $\ell_1$-norm, thus losing completely the structure encoded in $F$. 
%Indeed, in certain instances, the homogeneous convex envelope loses the combinatorial structure encoded in $F$. 

We discuss three examples that fall in this class of functions, where the non-homogeneous relaxation is tight while the homogeneous one is not.\\ %The following three examples fall in this class of functions.
%This effect was noted in \cite{obozinski2012convex} for the range function example, and in \cite{halabi2015totally} for the dispersive $\ell_0$-norm example.

\begin{example}[Range penalty] Consider $\text{range}(A) = \max(A) - \min(A) + 1$. For $F(A) = \text{range}(A)$, we have $F_-(A) = |A|$, while  $\tilde{F}_-=F$ by Prop. \ref{prop:LCENonHomo}.
%defined as $\text{range}(A) = \max(A) - \min(A) + 1$ where $\max(A)$ ($\min(A)$) denotes maximal (minimal) element in $A$.
% This penalty allow us to favor the selection of interval supports. It was shown in \cite{obozinski2012convex} that $F_-(A) = |A|$ for $F(A) = \text{range}(A)$. 
%of the range function is the cardinality. 
%Its homogeneous LCE is the cardinality function, while, by Proposition \ref{prop:LCENonHomo}, its non-homogeneous LCE is the range function itself.
\end{example}

\begin{example}[Dispersive $\ell_0$-penalty] \label{ex:Dispersive}
Given a set of predefined groups $\{G_1, \cdots, G_M\}$, consider the dispersive $\ell_0$-penalty, introduced by  \cite{halabi2015totally}: $F(A) = |A| + \iota_{B^T \1_A \leq \1}(A)$ where the columns of $B$ correspond to the indicator vectors of the groups, i.e., $B_{V,i} = \1_{G_i}$. The dispersive penalty enforces the selection of sparse supports where no two non-zeros are selected from the same group. Neural sparsity models induce such structures \cite{hegde2009compressive}. In this case, we  have $F_-(A) = |A|$, while $\tilde{F}_-=F$  by Prop. \ref{prop:LCENonHomo}.
%The homogeneous LCE of the dispersive $\ell_0$-penalty is also the cardinality, while by Proposition \ref{prop:LCENonHomo}, its non homogeneous LCE is itself.
\end{example}

\begin{example}[Weighted graph model] \label{ex:Graph}
Given a graph $\G = (V,E)$, consider a relaxed version of the weighted graph model of \cite{hegde2015nearly}: $F(A) = |A| + \iota_{\gamma(F_A) \leq g, \omega(F_A) \leq B }(A)$, where $\gamma(F_A) $ is the  number of connected components formed by the forest $F_A$ corresponding to $A$ and $ \omega(F_A) $ is the total weight of edges in the forest $F_A$. This model  describes a wide range of structures, including 1D-clustering, tree hierarchies, and the Earth Mover Distance model. We have $F_-(A) = |A|$, while $\tilde{F}_-=F$ by Prop. \ref{prop:LCENonHomo}.
%Its homogeneous LCE is also the cardinality function, while, by Proposition \ref{prop:LCENonHomo}, its non-homogeneous LCE is $F$ itself.
\end{example}

The last two examples belong to a natural class of structured sparsity penalties of the form $F(A) = |A| + \iota_{A \in \M}(A)$, which favors sparse non-zero patterns among a set $\M$ of allowed patterns. If $\M$ is down-monotone, i.e., $\forall A \in \M, \forall B \subseteq A, B \in \M$, then the non-homogeneous relaxation preserves its structure, i.e., $\tilde{F}_-=F$, while its homogeneous relaxation is oblivious to the hard constraints, with $F_-(A) = |A|$.

%Finally, if both are tight, and not identical, which is the case for example for p not infinity and submodular, it will depend on the problem

%% file: ConvOpt.tex
%!TEX root = AISTATS_2018_structure.tex
\section{Sparsity-inducing properties of convex relaxations}\label{sect:ContCond}
%monotone convex regularizers
The notion of LCE captures the combinatorial structure preserved by convex relaxations in a geometric sense.
%we cannot hope to enforce via convex relaxation. %We further investigate the combinatorial structure that can be enforced on solutions of learning problems regularized by these convex penalties, in the context of a well-specified model. 
In this section, we characterize the preserved structure from a statistical perspective. 
%We are interested characterizing the structure imposed by $\Omega_p$ and $\Theta_p$ on support of solutions of learning.
%hence study the statistical properties of convex monotone penalties in general, and in the following section we investigate these properties for the convex envelopes of $\ell_p$-regularized combinatorial functions.

% in particular.
%convex monotone penalties in general, and by the convex envelopes of $\ell_p$-regularized combinatorial functions in particular.
%assume that the linear model is well-specified,

To this end, we consider the linear regression model $y = X w^* + \epsilon$, where $X \in \R^{n \times d}$ is a fixed design matrix, $y \in \R^n$ is the response vector, and $\epsilon$ is a vector of iid random variables with mean $0$ and variance $\sigma^2$. 
%We consider a fixed design matrix $X \in \R^{n \times p}$ and $y \in \R^n$ a vector of random responses. 
Given $\lambda_n > 0$, we define $\hat{w}$ as a minimizer of the regularized least-squares:
 \begin{equation}\label{eq:GenEstimator}
\min_{w \in \R^d}  \frac{1}{2} \| y - Xw\|_2^2 + \lambda_n \Phi(w),
%\min_{w \in \R^d} L(w) - {z}^\top {w}  + \lambda \Phi(w)
\end{equation}
 We are interested in the sparsity-inducing properties of $\Omega_p$ and $\Theta_p$ on the solutions of \eqref{eq:GenEstimator}.
 In this section, we consider though the more general setting
where $\Phi$ is any proper normalized ($\Phi(0) = 0$) convex function which is absolute, i.e., $\Phi(w) = \Phi(|w|)$ and
monotonic in the absolute values of $w$, that is $|w| \geq |w'| \Rightarrow \Phi(w) \geq \Phi(w')$. 
In what follows, monotone functions refer to this notion of monotonicity.

%We determine in Section \ref{sect:NecCond} necessary conditions for a non-zero patterns to be allowed and in Section \ref{sect:SuffCond} sufficient conditions that lead to correct support recovery and estimation.

We determine in Section \ref{sect:NecCond} necessary conditions for support recovery in \eqref{eq:GenEstimator} and in Section  \ref{sect:SuffCond} we provide sufficient conditions for support recovery and consistency of a variant of \eqref{eq:GenEstimator}.
As both $\Omega_p$ and $\Theta_p$  are normalized absolute monotone convex functions, the results presented in this section apply directly to them as a corollary.

For simplicity, we  assume $Q = X^TX/n \succ 0$, thus $\hat{w}$ is unique. %, 
This forbids the high-dimensional setting. 
%and thus the minimizer of \eqref{eq:1StepEst} is unique, we denote it by
We expect though the insights developed towards the presented results to contribute to understanding the high-dimensional learning setting, which we defer to a later work. 

\subsection{Continuous stable supports}\label{sect:NecCond}

Existing results on the consistency of special cases of the estimator \eqref{eq:GenEstimator} typically rely heavily on decomposition properties of $\Phi$ \cite{negahban2011unified,  bach2010structured, obozinski2011group,  obozinski2012convex}. The notions of decomposability assumed in these prior works are either too strong or too specific to be applicable to the general convex penalties $\Omega_p$ and $\Theta_p$ we are considering. Instead, we introduce a general weak notion of decomposability applicable to any absolute monotone convex regularizer. \\%, and show that true

%Like regular analysis of sparsity-inducing norms, the analysis provided in this section relies heavily
%on decomposability properties of our norm Ω
%our regularizers are not decomposable nor norms, but we introduce a weak notion of decompoisability and show that the solutions of the above necessarily satisfyb this, and for sligthly stronger condition we introduce an esitmator that can recover

\begin{definition}[Decomposability]
Given $J \subseteq V$ and $w \in \R^d$, $\supp(w) \subseteq J$, we say that $\Phi$ is \emph{decomposable} at $w$ w.r.t $J$ if 
$\exists M_J>0$ such that $\forall \Delta \in \R^d, \supp(\Delta) \subseteq J^c,$
$$\Phi(w + \Delta) \geq  \Phi(w) + M_J \| \Delta\|_\infty.$$
\end{definition}

For example, for  the $\ell_1$-norm, this decomposability property holds for any $J \subseteq V$ and $w \in \R^d$, with $M_J = 1$. 

%this is not true for general strictly monotone functions, such as 
It is  reasonable to expect this property to hold at the solution $\hat{w}$ of \eqref{eq:GenEstimator} and its support $\hat{J} = \supp(\hat{w})$. Theorem \ref{them:NecessaryStableGeneral} shows that this is indeed the case. 
In Section \ref{sect:SuffCond}, we devise an estimation scheme able to recover supports $J$ that satisfy this property at \emph{any}  $w \in \R^d$. This leads then to following notion of \emph{continuous} stable supports, which characterizes supports with respect to the continuous penalty $\Phi$. In Section \ref{Sect:Disct}, we relate this to the notion of \emph{discrete} stable supports, which characterizes supports with respect to the combinatorial penalty $F$. \\
 
\begin{definition}[Continuous stability] \label{def:ContStability}
We say that $J \subseteq V$ is \emph{weakly stable} w.r.t $\Phi$ if \emph{there exists} $w \in \R^d$, $\supp(w) = J$ such that $\Phi$ is \emph{decomposable} at $w$ wrt $J$.
%$\exists M_J>0, \forall \Delta \in \R^d, \supp(\Delta) \subseteq J^c,$
%$$\Phi(w + \Delta) \geq  \Phi(w) + M_J \| \Delta\|_\infty.$$
%We say that $\Phi$ is \emph{decomposable} at $w$ w.r.t $J$.
Furthermore, we say that $J \subseteq V$ is \emph{strongly stable} w.r.t $\Phi$ if \emph{for all} $w \in \R^d$ s.t. $\supp(w) \subseteq J$, $\Phi$ is \emph{decomposable} at $w$ wrt $J$.
\end{definition}

%\begin{equation}\label{eq:GenEstimator}
%\min_{w \in \R^d} L(w) - {z}^\top {w}  + \lambda \Phi(w)
%\end{equation}
%where $L$ is a strongly-convex and smooth loss function and $z \in \R^d$ has a continuous density. 
%this includes least squares loss with str cvx design or least squares + lp norm squared.

Theorem \ref{them:NecessaryStableGeneral} considers slightly more general estimators than \eqref{eq:GenEstimator} and shows that weak stability is a necessary condition for a non-zero pattern to be allowed as a solution. \\

\begin{restatable}{theorem}{primeNecStableProp} \label{them:NecessaryStableGeneral}
%\begin{proposition}\label{prop:NecessaryStableGeneral}
The minimizer $\hat{w}$ of $\min_{w \in \R^d} L(w) - {z}^\top {w}  + \lambda \Phi(w)$, where $L$ is a strongly-convex and smooth loss function and $z \in \R^d$ has a continuous density w.r.t to the Lebesgue measure, has a weakly stable support w.r.t.~$\Phi$, with probability one. 
%\end{proposition}
\end{restatable}

This new result extends and simplifies the result in \cite{bach2010structured} which consideres the special case of quadratic loss functions and  $\Phi$ being the $\ell_\infty$-convex relaxation of a submodular function. The proof we present, in the Appendix, is also shorter and simpler.

\begin{corollary}\label{corr:NecessaryStable}
Assume $y \in \R^d$ has a continuous density w.r.t to the Lebesgue measure, %and $X^TX$ is invertible. 
then the support of the minimizer $\hat{w}$ of Eq. \eqref{eq:GenEstimator}
% is unique and its support $\supp(\hat{w})$ 
is weakly stable wrt $\Phi$ with probability one. 
\end{corollary}

\subsection{Adaptive estimation} \label{sect:SuffCond}

Restricting the choice of regularizers in \eqref{eq:GenEstimator} to convex relaxations as surrogates to combinatorial penalties is motivated by computational tractability concerns. 
However, other non-convex regularizers such as 
%can be preferable
%have been proposed in the literature. 
%For example, 
$\ell_\alpha$-quasi-norms \cite{knight2000asymptotics, frank1993statistical} or more generally penalties of the form $\Phi(w) = \sum_{i=1}^d \phi(|w_i|)$, where $\phi$ is a monotone concave penalty \cite{fan2001variable, daubechies2010iteratively, gasso2009recovering} can be more advantageous than the convex $\ell_1$-norm.
Such penalties are closer to the $\ell_0$-quasi-norm and penalize more aggressively small coefficients, thus they have a stronger sparsity-inducing effect than $\ell_1$-norm. 

The authors in \cite{jenatton2010structured} extended such concave penalties to the $\ell_\alpha /\ell_2$- quasi-norm $\Phi(w) = \sum_{i=1}^M \| w_{G_i}\|_\alpha$ for some $\alpha \in (0,1)$, which enforces sparsity at the group level more aggressively.
We generalize this to $\Phi(|w|^\alpha)$ where $\Phi$
is any structured sparsity-inducing monotone convex regularizer.

These non-convex penalties lead to intractable estimation problems, but approximate solutions can be obtained by majorization-minimization algorithms, as suggested for e.g.,  in \cite{figueiredo2007majorization, zou2008one, candes2008enhancing}.
%
%\begin{equation}\label{eq:NonConvexEst}
%\min_{w \in \R^d} \frac{1}{2} \| y - Xw\|_2^2 + \lambda_n \Phi(|w|^\alpha)
%\end{equation}
%where $\Phi$ 
\begin{restatable}{lemma}{primeMajoLem} \label{lem:Majorizer}
%\begin{lemma} \label{lem:Majorizer}
Let $\Phi$ be a monotone convex function,
$\Phi(|w|^\alpha)$ admits the following majorizer, $\forall w^0 \in \R^d$, $\Phi(|w|^\alpha) \leq (1-\alpha) \Phi(  |w^0|^\alpha) + \alpha \Phi(|w^0|^{\alpha-1} \circ |w|  )$, which is tight at $w^0$.
%\end{lemma}
\end{restatable}

We consider the adaptive weight estimator (\ref{eq:1StepEst}) resulting from applying a 1-step majorization-minimization to \eqref{eq:GenEstimator},
\begin{equation}\label{eq:1StepEst}
\min_{w \in \R^d} \frac{1}{2} \| y - Xw\|_2^2 + \lambda_n \Phi(|{w^0}|^{\alpha-1} \circ |w|),
\end{equation}
where ${w^0}$ is a $\sqrt{n}$-consistent estimator to $w^*$, that is converging to $w^*$ at rate $1/\sqrt{n}$ (typically obtained from $w^0 = \1$ or ordinary least-squares). 

%Note that the estimators \eqref{eq:1StepEst} are a subset of the estimators \eqref{eq:GenEstimator} considered in Sect. \ref{sect:NecCond}, and thus the same sufficient conditions for support recovery holds.
We study sufficient support recovery and estimation consistency conditions for (\ref{eq:1StepEst}) for general convex monotone regularizers $\Phi$. 
%To that end, we assume that the linear model is well-specified, with $y = X w^* + \epsilon$, where $\epsilon$ is a vector of i.i.d.~random variables with mean $0$ and variance $\sigma^2$. 
Such consistency results have been established for (\ref{eq:1StepEst}), in the classical asymptotic setting, only in the special case of $\ell_1$-norm in \cite{zou2006adaptive} and
%Sufficient support recovery and estimation consistency conditions 
for the (non-adaptive) estimator \eqref{eq:GenEstimator} for homogeneous convex envelopes of monotone submodular functions, for $p = \infty$ in \cite{bach2010structured} and for general $p$ in \cite{obozinski2012convex}, in the high dimensional setting, and for latent group Lasso norm in \cite{obozinski2011group}, in the asymptotic setting. 

Compared to prior works, the discussion of support recovery is complicated here by the fact that $\Phi$ is not necessarily a norm (e.g., if $\Phi = \Theta_p$) and only satisfies a weak notion of decomposability.

As in \cite{zou2006adaptive}, we consider the classical asymptotic regime in which the model generating the data is of fixed finite dimension $d$ while $n \to \infty$. As before, we assume $Q \succ 0$ and thus the minimizer of \eqref{eq:1StepEst} is unique, we denote it by $\hat{w}$. 

The following Theorem extends the results of \cite{zou2006adaptive} for the $\ell_1$-norm to any normalized absolute monotone convex regularizer if the true support satisfy the sufficient condition of strong stability in Definition \ref{def:ContStability}. As we previously remarked this condition is trivially satisfied for the $\ell_1$-norm.

%for the general structure-inducing penalties $\Omega_p$.

\begin{restatable}{theorem}{primeConsistThem}[Consistency and Support Recovery]\label{Thm:Consistency}
%\begin{theorem}[Consistency and Support Recovery]
Let $\Phi: \R^d \to \overline{\R}_+$ be a { proper normalized absolute monotone convex} function and denote by $J$ the true support $J = \supp(w^*)$.
If $|w^*|^\alpha \in \intr ~\dom ~\Phi$, $J$ is strongly stable with respect to $\Phi$ and $\lambda_n$ % = O(n^\gamma)$ for some $\gamma \in (\alpha/2, 1/2)$ and hence it 
satisfies $\frac{\lambda_n}{\sqrt{n}} \rightarrow 0, \frac{\lambda_n}{{n}^{\alpha/2}} \rightarrow \infty$,
then the estimator \eqref{eq:1StepEst} is consistent and asymptotically normal, i.e., it satisfies 
%asymptotic normality, i.e.:
\begin{equation}
\sqrt{n}(\hat{w}_{J} - w^*_{J}) \xrightarrow{d} \mathcal{N}(0, \sigma^2 Q_{J J}^{-1}), 
\end{equation}
and 
\begin{equation}
P(\supp(\hat{w}) = J) \rightarrow 1.
\end{equation}
%\end{theorem}
\end{restatable}
%The proof of consistency we present below follows a classical proof scheme (Bach, 2008a). However the originality of our work reside in that
%Such extension is non-trivial, since unlike l1-norm, the general regularizers considered are not separable nor strictly monotone. 
%Hence, the adaptive estimator \eqref{eq:1StepEst} is able to correctly identify any strongly stable support for any nomalized monotone convex regularizer. 
Consistency results in most existing works are established under various necessary conditions on $X$, some of which are difficult to verify in practice, such as the \emph{irrepresentability condition} (c.f., \cite{zou2006adaptive, bach2010structured, obozinski2011group, obozinski2012convex}). %, which is usually a necessary condition for 
 Adding data-dependent weights does not require such conditions and allows recovery even in the correlated measurement matrix setup as illustrated in our numerical results (c.f., Sect. \ref{sect:Simul}). 
%“most methods of high dimensional statistics”, e.g., Lasso (c.f., discussion in Section 2 of [37]). The numerical results illustrates this; adaptive weights allow for recovery even in the correlated measurement matrix setup.
%In particular, we are interested in structure-inducing regularizers that correspond to convex relaxations of combinatorial penalties.

%% file: DiscreteStable.tex
\section{Sparsity-inducing properties of combinatorial penalties}\label{Sect:Disct}

In section \ref{sect:ContCond}, we derived neccesary and sufficient conditions for support recovery defined with respect to the continuous convex penalties $\Omega_p$ and $\Theta_p$.
In this Section, 
%we study the sparsity-inducing properties of $\Omega_p$ and $\Theta_p$ . %convex envelopes of $\ell_p$-regularized combinatorial penalties $F_p$.
%Both penalties are normalized absolute monotone convex functions to which the necessary and sufficient conditions identified in Sections~\ref{sect:NecCond} and \ref{sect:SuffCond} apply.
we translate these to conditions with respect to the combinatorial penalties $F_p$ themselves. Hence, the results of this section allows one to check which supports to expect to recover, without the need to compute the corresponding convex relaxation.
To that end, we introduce in Section \ref{sec:DisStable} discrete counterparts of weak and strong stability, and show in Section \ref{sec:RelationDisceteConv} that discrete strong stability is a sufficient, and in some cases necessary, condition for support recovery.

\subsection{Discrete stable supports}
 \label{sec:DisStable}

We recall the concept of discrete stable sets \cite{bach2010structured}, also referred to as \emph{flat} or \emph{closed} sets \cite{krause2012near}. We refer to such sets as discrete weakly stable sets and introduce a stronger notion of discrete stability.

\begin{definition}[Discrete stability] \label{def:DisStable}
Given a monotone set function $F: 2^V \to \overline{\R}_+$, a set $J \subseteq V$ is said to be \emph{weakly stable} w.r.t $F$ if $\forall i \in J^c, F(J \cup  \{i\}) > F(J)$.\\
A set $J \subseteq V$ is said to be \emph{strongly stable} w.r.t $F$ if $\forall A \subseteq J, \forall i \in J^c, F(A \cup  \{i\}) > F(A)$.
\end{definition}

Note that discrete stability imply in particular feasibility, i.e., $F(J) < + \infty$. Also, if $F$ is a strictly monotone function, such as the cardinality function, then all supports are stable w.r.t $F$.
It is interesting to note that for monotone submodular functions, weak and strong stability are equivalent. In fact, this equivalence holds for a more general class of functions, we call $\rho$-submodular.
\begin{definition}
A function $F: 2^V \to \R$ is $\rho$-submodular iff $\exists \rho \in (0,1]$ s.t., $\forall B \subseteq V, A \subseteq B, i \in B^c$
$$ \rho [ F(B \cup \{i\}) - F(B) ] \leq F(A \cup \{i\}) - F(A) $$
\end{definition}
The notion of $\rho$-submodularity is a special case of the weakly DR-submodular property defined for continuous functions \cite{hassani2017gradient}. It
is also related to the notion of weak submodularity (c.f., \cite{das2011submodular, elenberg2016restricted}). We show in the appendix that $\rho$-submodularity is a stronger condition. % than weak submodularity.\\
%another notion of ``approximate" submodularity, 
\begin{restatable}{proposition}{primeEqDefStableProp}\label{prop:EqDefStableProp}
%\begin{proposition}
If $F$ is a finite-valued monotone function, $F$ is $\rho$-submodular iff discrete weak stability is equivalent to strong stability.
%\end{proposition}
\end{restatable}

\begin{example}
The range function $\text{range}(A) = \max(A) - \min(A) + 1$ is $\rho$-submodular with $\rho = \frac{1}{d-1}$. 
%which supports stable w.r.t $\Omega_p$ are all supports.
\end{example}

\subsection{Relation between discrete and continuous stability} \label{sec:RelationDisceteConv}

This section provides several technical results relating the discrete and continuous notions of stability. It thus provides us with the necessary tools to characterize which supports can be correctly estimated w.r.t the combinatorial penalty itself, without going through its relaxations. %In particular, it establish that discrete strong stability is a sufficient condiot

\begin{restatable}{proposition}{primeDSCSProp}
\label{prop:DS-CS}
%\begin{proposition}
Given any {monotone} set function $F$, all sets $J \subseteq V$ {strongly stable w.r.t to $F$} are also strongly stable w.r.t $\Omega_p$ and $\Theta_p$.
 %I don't think we need this: containing the support of $w$, 
%$$\exists M>0, \Omega(|w_J| + \delta \1_i ) \geq  \Omega(|w_J|) + \delta M, \forall \delta>0, \forall i \in J^c, \forall w \in \R^d$$
%\end{proposition}
\end{restatable}

%\begin{restatable}{proposition}{primeStabNonHomoProp}\label{prop:StableNonHomoLp}
%%\begin{proposition}
%Given any {monotone} set function $F$, and its corresponding $\ell_p$-non-homogeneous convex envelope $\Theta_p$, all sets $J \subseteq V$ {strongly stable w.r.t to $F$} are also strongly stable w.r.t $\Theta_p$.
%%$$\exists M>0, \Theta_p(|w_J| + \delta \1_i ) \geq  \Theta_p(|w_J|) + \delta M, \forall \delta>0, \forall i \in J^c, \forall w \in \R^d$$
%%\end{proposition}
%\end{restatable}
It follows then by Theorem \ref{Thm:Consistency} that discrete strong stability is a sufficient condition for correct estimation.

\begin{corollary}\label{cor:Consistent}
If $\Phi$ is equal to $\Omega_p$ or $\Theta_p$ for $p \in (1,\infty)$ and $\supp(w^*) = J$ is strongly stable w.r.t $F$, then Theorem \ref{Thm:Consistency} holds, i.e., the adaptive estimator \eqref{eq:1StepEst} is consistent and correctly recovers the support. This also holds for $p = \infty$ if we further assume that $\| w^* \|_\infty <1$.
%strongly stable assume indirectly that F(suppp(J)) < + inf
\end{corollary}
Furthermore, if $F$ is $\rho$-submodular, then by Proposition \ref{prop:EqDefStableProp}, it is enough for $\supp(w^*) = J$ to be weakly stable w.r.t $F$ for Corollary \ref{cor:Consistent} to hold. Conversely, Proposition \ref{prop:CS-DSpinf} below shows that discrete strong stability is also a necessary condition for continuous strong stability, in the case where $p = \infty$ and $F$ is equal to its LCE. \\
 
\begin{restatable}{proposition}{primeCSDSpinfProp} \label{prop:CS-DSpinf}
%\begin{proposition} \label{lem:CS-DSpinf}
If $F = {F}_-$ and $J$ is strongly stable w.r.t $\Omega_\infty$, then $J$ is strongly stable w.r.t $F$.
Similarly, for any monotone $F$, if $J$ is strongly stable w.r.t $\Theta_\infty$, then $J$ is strongly stable w.r.t $F$.
%(if definition is for $\supp(w) \subseteq J$, otherwise it's weakly) 
%\end{proposition}
\end{restatable}

Finally, in the special case of monotone submodular function, the following Corollary \ref{corr:SubStEq}, along with Proposition \ref{prop:DS-CS} demonstrates that all definitions of stability become equivalent. 
We thus recover the result in \cite{bach2010structured} showing that discrete weakly stable supports correspond to the set of allowed sparsity patterns for monotone submodular functions. 

\begin{restatable}{corollary}{primeSubStEqCorr} \label{corr:SubStEq}
%\begin{corollary}
If $F$ is monotone submodular and $J$ is weakly stable w.r.t $\Omega_\infty = \Theta_\infty$ then $J$ is weakly stable w.r.t $F$. %Hence all definitions of stability are equivalent for monotone submodular functions.
%\end{corollary}
\end{restatable}

\subsection{Examples}\label{sect:ex}

We highlight in this section what are the supports recovered by the adaptive estimator (AE) \eqref{eq:1StepEst} with the homogeneous convex relaxation $\Omega_p$ and non-homogeneous convex relaxation $\Theta_p$ of some examples of structure priors. For simplicity, we will focus on the case $p=\infty$. Also in all the examples we consider below, weak and strong discrete stability are equivalent, so we omit the weak/strong specification. 
Note that it is desirable that the regularizer used enforces the recovery of only the non-zero patterns satisfying the desired structure.
%For that we rely on results from Theorem 1 and Propositions 3-4-5.

\textbf{Monotone submodular functions:}  As discussed above, for this class of functions, all stability definitions are equivalent and $\Omega_\infty = \Theta_\infty = f_L$. As a result, AE recovers any discrete stable non-zero pattern. This includes the following examples (c.f., \cite{obozinski2012convex} for further examples).
\begin{itemize}
\item \textbf{Cardinality:} As a strictly monotone function,  all supports are stable w.r.t to it.  Thus AE recovers \emph{all} non-zero patterns with $\Omega_\infty$ and $\Theta_\infty$, given by the $\ell_1$-norm.
%and both its homogeneous and non-homogeneous relaxation, given by the $\ell_1$-norm, are strictly monotone, hence all sets are stable (strongly and weakly) w.r.t to them.
\item \textbf{Overlap count function:} $F_\cap(A) = \sum_{G \in \G, G \cap A \not = \emptyset} d_G $ where $\G$ is a collection of predefined groups $G$ and $d_G$ their associated weights. $\Omega_\infty$ and $\Theta_\infty$ are given by the $\ell_1 / \ell_\infty$-group Lasso norm, and stable patterns are complements of union of groups. For example, for hierarchical groups (i.e., groups consisting of each node and its descendants on a tree), AE recovers rooted connected tree supports.
\item \textbf{Modified range function:} The range function can be transformed into a submodular function, if scaled by a constant as suggested in \cite{bach2010structured}, yielding the monotone submodular function $F^{\text{mr}}(A) = d - 1 + \text{range}(A), \forall A \not = \emptyset$ and $F^{\text{mr}}(\varnothing) = 0$. This can actually be written as an instance of $F_\cap$ with groups defined as $\G = \{[1,k] : 1 \leq k \leq d\} \cup \{[k,d] : 1 \leq k \leq d\}$. 
 %These sets of groups are illustrated in Figure \red{Can we use figure from \cite{obozinski2012convex}?}.
This norm was proposed to induce interval patterns by \cite{jenatton2011structured}, and indeed its stable patterns are interval supports. We will compare this function in the experiments with the direct convex relaxations of the range function.  %(c.f., Sect. \ref{sect:Simul}).
% as a workaround for the loss of structure by the homogeneous relaxation 
%and shown to be the convex envelope of $F^{\text{mr}}$ in \cite{bach2010structured}. 

\end{itemize}

\textbf{Range function:}  The range function is $\frac{1}{d-1}$-submodular, thus its discrete strongly and weakly stable supports are identical and they correspond to interval supports. As a result, AE recovers interval supports with $\Theta_\infty$. On the other hand, since the homogeneous LCE of the range function is the cardinality,  AE recovers all supports with $\Omega_\infty$.
% Since the range function is monotone, then by Proposition \ref{prop:CS-DSpinf}, sets strongly stable w.r.t its non-homogeneous convex envelope $\Theta^{r}_\infty$ are interval supports too. On the other hand, its homogeneous convex envelope $\Omega^{r}_\infty = \| \cdot \|_1$ admits all sets as strongly stable.

\textbf{Down monotone structures:} Functions of the form $F(A) = |A| + \iota_{A \in \M}(A)$, where $\M$ is down-monotone, also have their discrete strongly and weakly stable supports identical and given by the feasible set $\M$. 
These structures include the dispersive and graph models discussed in examples \ref{ex:Dispersive} and \ref{ex:Graph}. 
Since their homogeneous LCE is also the cardinality, then AE recovers all supports with $\Omega_\infty$, and only feasible supports with $\Theta_\infty$.

%% file: Simulations.tex
\section{Numerical Illustration}\label{sect:Simul}
%\vspace{-20pt}

\begin{figure}\label{fig:toyexample}
\vspace{-20pt} 
\centering
\begin{tabular}{c c}\hspace{-15pt}
\includegraphics[trim=50 180 50 150, clip, scale=.24]{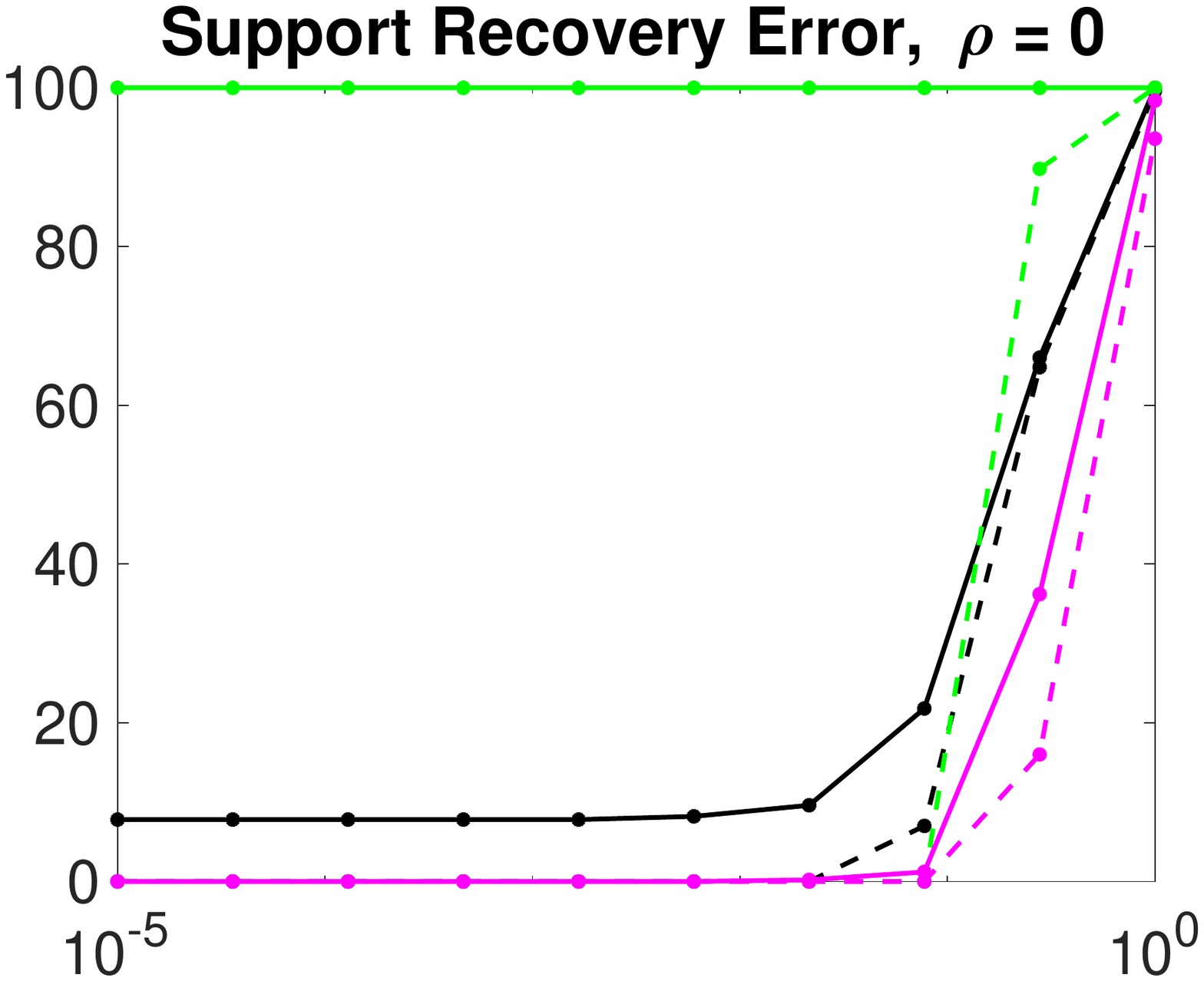} & \hspace{-15pt}
\includegraphics[trim=40 180 50 150, clip, scale=.24]{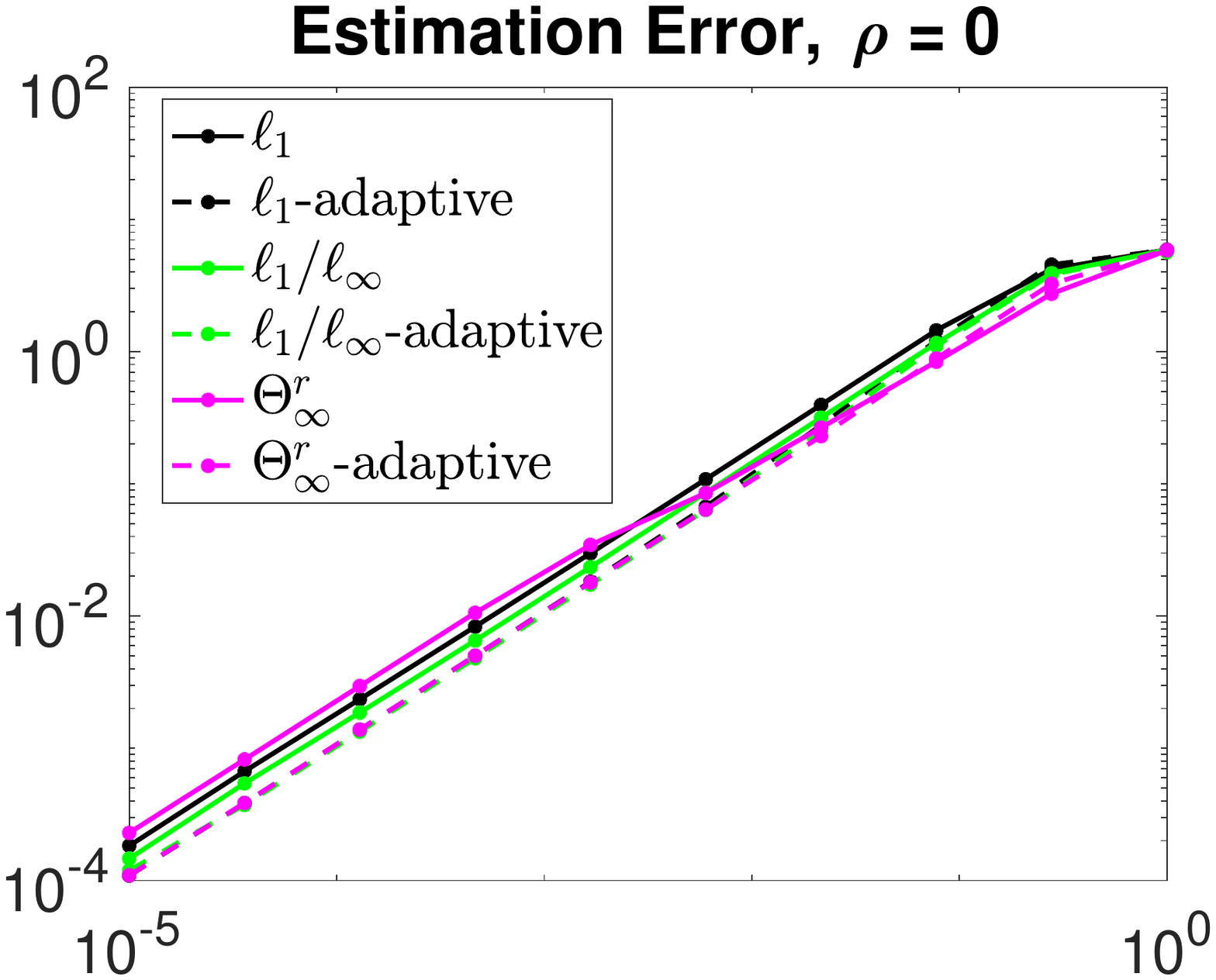} \\
\hspace{-15pt}
\includegraphics[trim=50 180 50 150, clip, scale=.24]{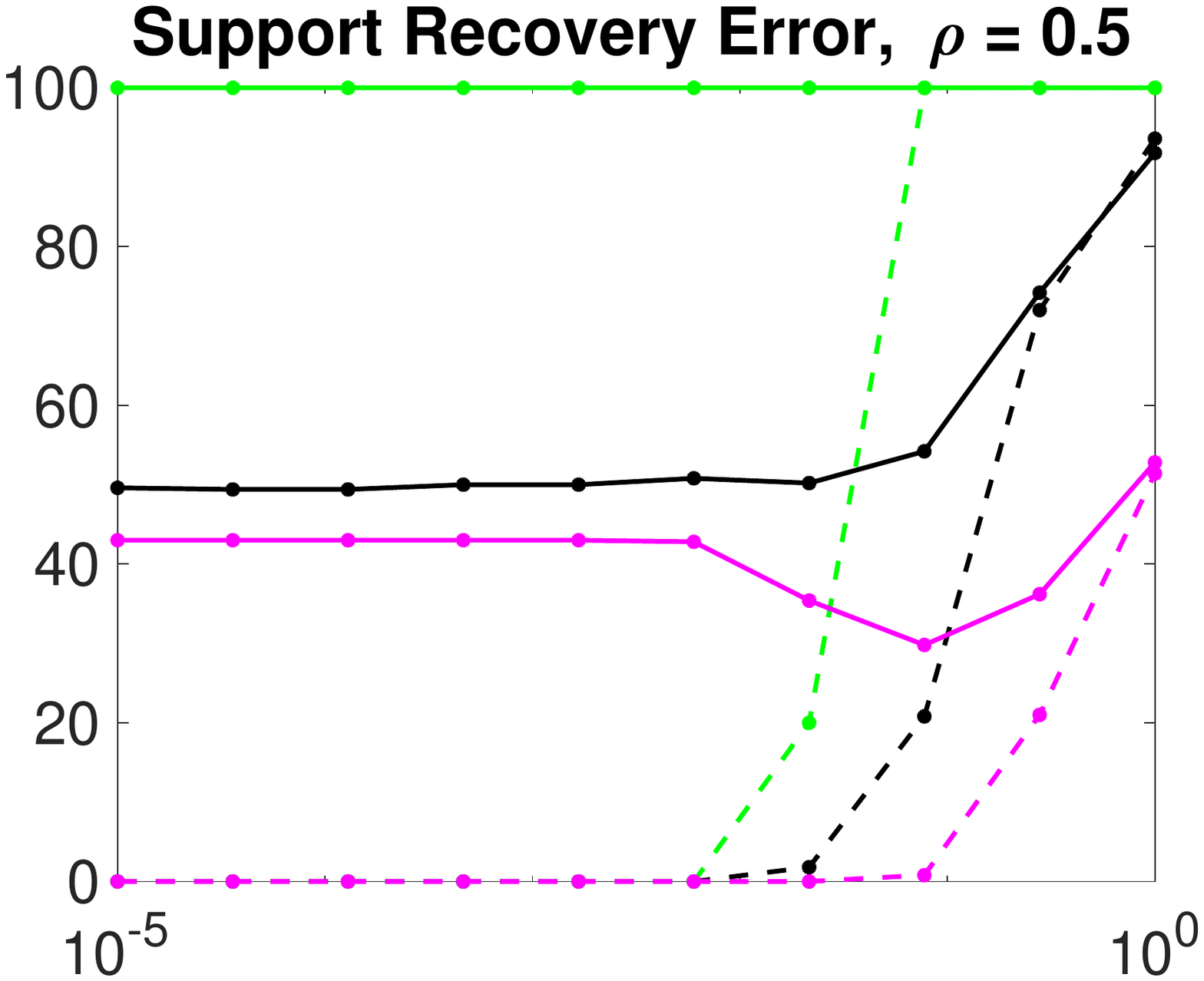} & \hspace{-15pt}
\includegraphics[trim=40 180 50 150, clip, scale=.24]{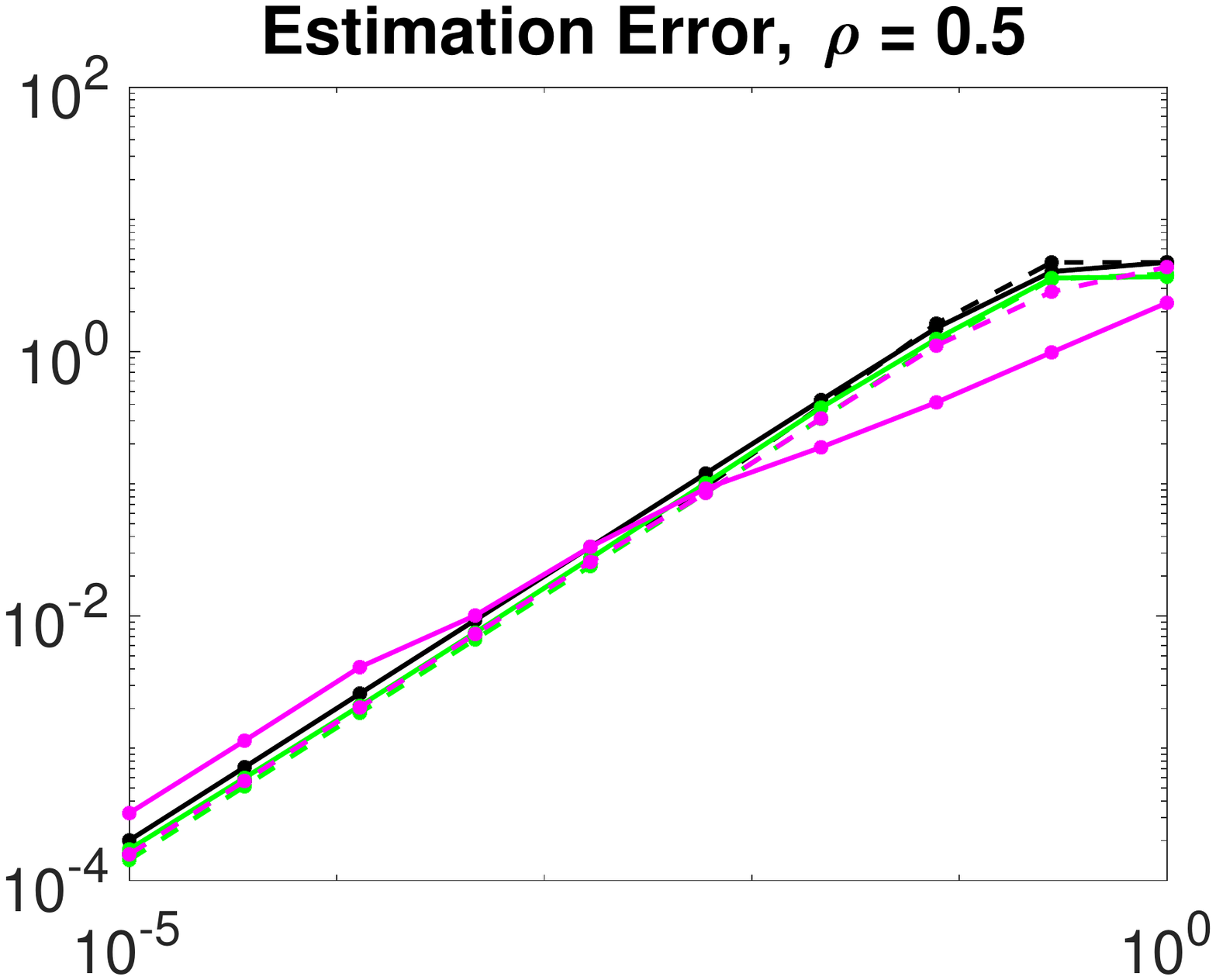} 
\end{tabular}
\caption{ (Left column) Best Hamming distance and (Right column) best least square error to the true vector $w^*$, along the regularization path, averaged over 5 runs.  }
\end{figure}

To illustrate the results presented in this paper, we consider the problem of estimating the support of a parameter vector $w^* \in \R^d$ whose support is an interval. It is natural then to choose as combinatorial penalty the range function whose stable supports are intervals. We aim to study the effect of adaptive weights, as well as the effect of the choice of homogeneous vs. non-homogeneous convex relaxation for regularization, on the quality of support recovery.

As discussed in Section \ref{sect:ex}, the $\ell_\infty$-homogeneous convex envelope of the range is simply the $\ell_1$-norm. 
Its $\ell_\infty$-non-homogeneous convex envelope $\Theta^{r}_\infty$ can be computed using the formulation \eqref{eq:VarLpLinfNonHomo}, where only interval sets need to be considered in the constraints, leading to a quadratic number of constraints.
We also consider the $\ell_1 / \ell_\infty$-norm that corresponds to the convex relaxation of the modified range function $F^{\text{mr}}$.

We consider a simple regression setting in which $w^* \in \R^d$ is a constant signal whose support is an interval. The choice of $p = \infty$ is well suited for constant valued signals.
The design matrix $X \in \R^{d \times n}$ is either drawn as  (1) an i.i.d Gaussian matrix with normalized columns, or  (2) a correlated Gaussian matrix with normalized columns, with the off-diagonal values of the covariance matrix set to a value $\rho = 0.5$. We observe noisy linear measurements $y = X w^* + \epsilon$, where the noise vector is i.i.d.~with variance $\sigma^2$, where $\sigma$ is varied between $10^{-5}$ and $1$. We solve problem \eqref{eq:1StepEst} with and without adaptive weights $|w^0|^{\alpha - 1}$, where $w^0$ is taken to be the least squares solution and $\alpha = 0.3$.

We assess the estimators obtained through the different regularizers both in terms of support recovery and in terms of estimation error. Figure \ref{fig:toyexample} plots (in logscale) these two criteria against the noise level~$\sigma$. We plot the best achieved error on the regularization path, where the regularization parameter $\lambda$ was varied between $10^{-6}$ and $10^3$. We set the parameters to $d = 250, k = 100, n = 500$.

We observe that the adaptive weight scheme helps in support recovery, especially in the correlated design setting. Indeed, Lasso is only guaranteed to recover the support under an ``irrepresentability condition" \cite{zou2006adaptive}. This is satisfied with high probability only in the non-correlated design. On the other hand, adaptive weights allow us to recover any strongly stable support, without any additional condition, as shown in Theorem \ref{Thm:Consistency}.
The $\ell_1 / \ell_\infty$-norm performs poorly in this setup. In fact, the modified range function $F^{\text{mr}}$, introduced a gap of $d$ between non-empty sets and the empty set. This leads to the undersirable behavior, already documented in \cite{bach2010structured, jenatton2011structured} of adding all the variables in one step, as opposed to gradually.
Adaptive weights seem to correct for this effect, as seen by the significant improvement in performance. 
Finally, note that choosing the ``tighter" non-homogeneous convex relaxation leads to better support recovery. Indeed, $\Theta^{r}_\infty$ performs better than $\ell_1$-norm in all setups.

%% file: Appendix.tex
\onecolumn
\section{Appendix}

\subsection{Variational forms of convex envelopes ({Proof of lemma \ref{lem:NonHomEnv} and Remark \ref{rmk:SubNonHomoEnv})} }

In this section, we recall the different variational forms of the homogeneous convex envelope derived in \cite{obozinski2012convex} and derive similar variational forms for the non-homogeneous convex envelope, which includes the ones stated in lemma \ref{lem:NonHomEnv}). These variational forms will be needed in some of our proofs below.

\begin{lemma}
The homogeneous convex envelope $\Omega_p$ of $F_p$ admits the following variational forms.
\begin{align}
\Omega_\infty(w) &=  \min_{\alpha} \{ \sum_{S \subseteq V} \alpha_S F(S) :  \sum_{S \subseteq V} \alpha_S \1_S \geq |w|, \alpha_S \geq 0\}.
\label{eq:ConvCoverA}\\
%&=  \min_{v} \{ \sum_{S \subseteq V} F(S) \| v^S \|_\infty  :  \sum_{S \subseteq V} v^S = |w|, \supp(v^S) \subseteq S\}. \label{eq:InfLGLgen}\\
\Omega_p(w) &=  \min_{v} \{ \sum_{S \subseteq V} F(S)^{1/q} \| v^S \|_p :  \sum_{S \subseteq V} v^S = |w|, \supp(v^S) \subseteq S\}. \label{eq:LpLGLgen}\\
&= \max_{\kappa \in \R^d_+}  \sum_{i = 1}^d \kappa_i^{1/q} |w_i| \text{ s.t. } \kappa(A) \leq F(A), \forall A \subseteq V.  \label{eq:SupportForm}\\
&=\inf_{\eta \in \R^d_+}  \frac{1}{p} \sum_{j=1}^d \frac{|w_j|^{p}}{\eta_j^{p -1}} + \frac{1}{q} \Omega_\infty(\eta). \label{eq:VarLpLinfA}  
%\Omega^*_p(s) &= \max_{A \not = \emptyset} \frac{\|s_A\|_q}{F(A)^{1/q}} \label{eq:DualNorm}
\end{align}
\end{lemma}

The non-homogeneous convex envelope of a set function $F$, over the unit $\ell_\infty$-ball was derived in \cite{halabi2015totally}, where it was shown that $\Theta_\infty(w) = \inf_{\eta \in [0,1]^d} \{ f(\eta) : \eta \geq |w|\}$ where $f$ is any proper, l.s.c. convex \emph{extension} of $F$ (c.f., Lemma 1 \cite{halabi2015totally}). A natural choice for $f$ is the \emph{convex closure} of $F$, which corresponds to the \emph{tightest} convex extension of $F$ on $[0,1]^d$. We recall the two equivalent definitions of convex closure, which we have adjusted to allow for infinite values.

%For set functions that take $\infty$ value, we extend definition \ref{def:convClosureExp} by replacing the $\min$ by $\inf$, to have $f^-(\1_S) = +\infty$ iff $F(S) = + \infty$.

\begin{definition}[Convex Closure; c.f., {\cite[Def. 3.1]{dughmi2009submodular}}]\label{def:convClosure}
Given a set function $F : 2^V \to \overline{\R}$, the convex closure $f^{-}: [0,1]^d \to \overline{\R}$ is the point-wise largest convex function from $[0,1]^d$ to $\overline{\R}$ that always lowerbounds $F$.
\end{definition}

\begin{definition}[Equivalent definition of Convex Closure; c.f., {\cite[Def. 1]{Vondrak2010}} and {\cite[Def. 3.2]{dughmi2009submodular}}] \label{def:convClosureExp}
Given any set function $f: \{0,1\}^n \rightarrow \overline{\R}$, the convex closure of $f$ can equivalently be defined $\forall w \in [0,1]^n$  as:
\[f^-(w) = \inf \{ \sum_{S \subseteq V} \alpha_S F(S) : w= \sum_{S \subseteq V} \alpha_S \1_S, \sum_{S \subseteq V} \alpha_S =1, \alpha_S \geq 0\}\]
\end{definition}
It is interesting to note that $f^-(w) = f_L(w)$ where $f_L$ is Lov\'asz extension iff $F$ is a submodular function \cite{Vondrak2010}. 

The following lemma  derive variational forms of $\Theta_p$ for any $p \geq 1$ that parallel the ones known for the homogeneous envelope.

\begin{lemma}\label{lem:NonHomVarForms}
The non-homogeneous convex envelope $\Theta_p$ of $F_p$ admits the following variational forms.
\begin{align}
 \Theta_\infty(w)  &= \inf \{ \sum_{S \subseteq V} \alpha_S F(S) :  \sum_{S \subseteq V} \alpha_S \1_S \geq |w|, \sum_{S \subseteq V} \alpha_S =1, \alpha_S \geq 0\}.
 \label{eq:ConvCoverNonHomoA}\\ %, \forall w \in [-1,1]^d
  %no need to restrict w more, if w is such that no s covers it in the domain of f-, then the inf will be infinity, so this is well defined
%&=  \inf_{v} \{ \sum_{S \subseteq V} F(S) \| v^S \|_\infty  :  \sum_{S \subseteq V} v^S = |w|,  \sum_{S \subseteq V} \| v^S \|_\infty =1, \supp(v^S) \subseteq S\}. \label{eq:LinfLGLgenNonHomo}
\Theta_p(w) %&=  \text{\red{TBD: LGL form}}. \label{eq:LpLGLgenNonHomo}\\
&= \max_{\kappa \in \R^d}  \sum_{j=1}^d \psi_j(\kappa_j,w_j)  +  \min_{S \subseteq V} F(S) - \kappa(S), ~\forall w \in \dom(\Theta_p(w)).   \label{eq:SupportFormNonHomo}\\
&=\inf_{\eta \in [0,1]^d}  \frac{1}{p} \sum_{j=1}^d \frac{|w_j|^{p}}{\eta_j^{p -1}} + \frac{1}{q} f^-(\eta),
\label{eq:VarLpLinfNonHomoA}  
\end{align}
where $\dom(\Theta_p) = \{ w | \exists \eta \in [0,1]^d \text{ s.t } \supp(w) \subseteq \supp(\eta), \eta \in \dom(f^-)\}$, and
where we define
\begin{align*}
\psi_j(\kappa_j,w_j)  &:=\begin{cases} \kappa_j^{1/q}|w_j| &\text{ if $|w_j| \leq \kappa_j^{1/p}, \kappa_j \geq 0$}\\
\frac{1}{p}  |w_j|^{p}+ \frac{1}{q}   \kappa_j  &\text{otherwise.}
\end{cases} 
\end{align*} 
 If $F$ is monotone, $\Theta_\infty = f^-$, then we can replace $f^-$ by $\Theta_\infty$ in \eqref{eq:VarLpLinfNonHomoA} and we can restrict $\kappa \in \R_+^d$ in \eqref{eq:SupportFormNonHomo}.
\end{lemma}

To prove the variational form \eqref{eq:ConvCoverNonHomoA} in Lemma \ref{lem:NonHomVarForms}, we need to show first the following property of $f^-$.

\begin{proposition}[c.f., {\cite[Prop. 3.23]{dughmi2009submodular}} ] \label{prop:minClosure}
The minimum values of a proper set function $F$ and its convex closure $f^{-}$ are equal, i.e., 
\[ \min_{w \in [0,1]^d} f^{-}(w) = \min_{S \subseteq V} F(S)\]
If $S$ is a minimizer of $f(S)$, then $\1_S$ is a minimizer of $f^{-}$. Moreover, if $w$ is a minimizer of $f^{-}$, then every set in the support of $\alpha$, where $f^{-}(w) = \sum_{S \subseteq V} \alpha_S F(S)$, is a minimizer of $F$.
\end{proposition}
\begin{proof}
%F is assumed to be proper so the minimum always exists
First note that, $\{0,1\}^d \subseteq [0,1]^d$ implies that $f^{-}(w^*) \leq F(S^*)$. On the other hand, $f^-(w^*) = \sum_{S \subseteq V} \alpha^*_S F(S) \geq \sum_{S \subseteq V} \alpha^*_S F(S^*) =  F(S^*)$. The rest of the proposition follows directly.
\end{proof}

Given the choice of the extension $f = f^-$, the variational form \eqref{eq:ConvCoverNonHomoA} of $\Theta_\infty$ given in lemma \ref{lem:NonHomVarForms} follows directly from definition \ref{def:convClosureExp} and proposition \ref{prop:minClosure}, as shown in the following corollary.

\begin{corollary}\label{corr:varformConvEnvelope}
Given any set function $F:2^V \to \overline{\R}_+$ and its corresponding convex closure $f^-$, the convex envelope of $F(\supp(w))$ over the unit $\ell_\infty$-ball is given by 
\begin{align*}
\Theta_\infty(w) &= \inf_\alpha \{ \sum_{S \subseteq V} \alpha_S F(S) :  \sum_{S \subseteq V} \alpha_S \1_S \geq |w|, \sum_{S \subseteq V} \alpha_S =1, \alpha_S \geq 0\}.
\\
&=  \inf_{v} \{ \sum_{S \subseteq V} F(S) \| v^S \|_\infty  :  \sum_{S \subseteq V} v^S = |w|,  \sum_{S \subseteq V} \| v^S \|_\infty =1, \supp(v^S) \subseteq S\}.
 \end{align*}
% for all  $w$ such that $\exists s \in [0,1]^d \cap \dom(f^-), s \geq |w|$, $\infty$ otherwise.
\end{corollary}
\begin{proof}
$f^-$ satisfies the first 2 assumptions required in Lemma 1 of \cite{halabi2015totally}, namely, $f^{-}$ is a lower semi-continuous convex extension of $F$ which satisfies
\[ \max_{S \subseteq V} m(S) - F(S) = \max_{w \in [ 0,1]^d } m^Tw - f^-(w), \forall m \in \R_+^d\]
To see this note that
 $m^Tw^* - f^-(w^*) = \sum_{S \subseteq V} \alpha^*_S (m^T\1_S - F(S) )\geq \sum_{S \subseteq V} \alpha^*_S (m^T\1_{S^*} - F(S^*)) =  m(S^*) - F(S^*)$. The other inequality is trivial.
The corollary then follows directly from Lemma 1 in \cite{halabi2015totally} and definition \ref{def:convClosureExp}.
\end{proof}
Note that $\dom(\Theta_\infty) = \{w: \exists \eta \in [0,1]^d \cap \dom(f^-), \eta \geq |w| \}$. Note also that $\Theta_\infty$ is monotone even if $F$ is not. On the other hand, if $F$ is monotone, then $f^-$ is monotone on $[0,1]^d$ and $\Theta_\infty(w) = f^-(|w|)$.
Then the proof of remark \ref{rmk:SubNonHomoEnv} follows, since
if $F$ is a monotone submodular function and $f_L$ is its Lov\'asz extension, then  $\Theta_\infty(w)  = f^-(|w|) = f_L(|w|) = \Omega_\infty(w), \forall w \in [-1,1]^d$, where the last equality was shown in \cite{bach2010structured}. 

Next, we derive the convex relaxation of $F_p$ for a general $p \geq 1$.
\begin{proposition}\label{prop:LpregNonHomo}
Given any set function $F:2^V \to \overline{\R}_+$ and its corresponding convex closure $f^-$, the convex envelope of $F_{\mu \lambda}(w) = \mu F(\supp(w)) + \lambda \| w \|_p^p$ is given by 
\begin{align*}
\Theta_p(w) =\inf_{\eta \in [0,1]^d} \lambda \sum_{j=1}^d \frac{|w_j|^{p}}{\eta_j^{p -1}} + \mu f^-(\eta).
\end{align*}
Note that $\dom(\Theta_p) = \{ w | \exists \eta \in [0,1]^d \text{ s.t } \supp(w) \subseteq \supp(\eta), \eta \in \dom(f^-)\}$. %\supseteq  \dom(\Theta_\infty)
\end{proposition} 
\begin{proof}
Given any proper l.s.c. convex extension $f$ of $F$, we have: 

First for the case where $p=1$:
\begin{align*}
F_{\mu \lambda}^\ast(s) &= \sup_{w \in \R^n} w^Ts -\mu  F(\supp(w)) - \lambda \| w \|_1\\
  &= \sup_{\eta \in \{ 0,1\}^d} \sup_{\scriptsize \colcst{\mathds{1}_{\supp(w)} = \eta \\  \sign(w) = \sign(s)}}  |w|^T (|s|- \lambda \1)  - \mu F(\eta) \\
  &= \iota_{\{|s| \leq \lambda \1\} }(s) - \inf_{\eta \in \{ 0,1\}^d} \mu F(\eta).
\end{align*}
Hence $F_{\mu \lambda}^{\ast \ast}(w)= \lambda  \| w \|_1 +  \inf_{\eta \in \{ 0,1\}^d} \lambda F(\eta)$. 
For the case $p \in (1,\infty)$.
\begin{align*}
F_{\mu \lambda}^\ast(s) &= \sup_{w \in \R^d} w^T s - \mu F(\supp(w)) - \lambda \| w \|_p^p\\
  &= \sup_{\eta \in \{ 0,1\}^d} \sup_{\scriptsize \colcst{\mathds{1}_{\supp(w)} = \eta \\  \sign(w) = \sign(s)}} |w|^T |s|-\lambda \| w \|_p^p  - \mu F(\eta) \\
  &= \sup_{\eta \in \{ 0,1\}^d} \frac{\lambda (p-1)}{(\lambda p )^q} \eta^T|s|^q- \mu F(\eta) \tag{$|s_i|=\lambda p |x^*_i|^{p-1}, \forall \eta_i \not = 0$}\\
  &= \sup_{\eta \in [0,1]^d} \frac{\lambda (p-1)}{(\lambda p )^q} \eta^T|s|^q - \mu f^-(\eta) .
\end{align*}
We denote $\hat{\lambda} = \frac{\lambda (p-1)}{(\lambda p )^q}$.
\begin{align*}
F_{\mu \lambda}^{\ast \ast}(w) &= \sup_{s \in \R^d} w^T s - F_{\mu \lambda}^{\ast}(s) \\
 &= \sup_{s\in \R^d }  \min_{\eta \in [0,1]^d}  s^T w - \hat{\lambda} \eta^T| s|^q + \mu f^-(\eta)  \\
  &\stackrel{\star}{=}  \inf_{\eta \in [0,1]^d} \sup_{\scriptsize \colcst{s \in \R^p \\ \sign(s) = \sign(w)}}   |s|^T |w| -  \hat{\lambda} \eta^T| s|^q + \mu f^-(\eta)  \\
 &=\inf_{\eta \in [0,1]^d} \lambda ( |w|^p)^T \eta^{1-p} + \mu f^-(\eta),
 \end{align*}
 where the last equality holds since $|w_i|=\hat{\lambda} \eta_i q |s^*_i|^{q-1}, \forall \eta_i \not = 0$, otherwise $s^*_i = 0$ if $w_i = 0$ and $\infty$ otherwise.
 $(\star)$ holds by Sion's minimax theorem \cite[Corollary 3.3]{sion1958general}. 
Note then that the minimizer $\eta^\ast$ (if it exists) satisfies $\supp(w) \subseteq \supp(\eta^*)$. 
Finally, note that if we take the limit as $p \to \infty$, we recover $\Theta_{\infty}= \inf_{
\eta \in [0,1]^d}\{ f^-(\eta) : \eta \geq |x| \}$.
\end{proof}

The variational form \eqref{eq:VarLpLinfNonHomoA} given in lemma \ref{lem:NonHomVarForms} follows from proposition \ref{prop:LpregNonHomo} for the choice $\mu = \frac{1}{q}, \lambda = \frac{1}{p}$.

The following proposition derives the variational form \eqref{eq:SupportFormNonHomo} for $p = \infty$.

\begin{proposition}\label{prop:SuppFctNonHomo}
Given any set function $F:2^V \to \R \cup \{+\infty\}$, and its corresponding convex closure $f^-$, $\Theta_\infty$ can be written $\forall w \in \dom(\Theta_\infty)$ as
\begin{align*}
\Theta_\infty(w)  &= \max_{\kappa \in \R^d_+}  \{  \kappa^T |w| + \min_{S \subseteq V} F(S) - \kappa(S) \}  \\
 &= \max_{\kappa \in \R^d_+}  \{  \kappa^T |w| + \min_{S \subseteq \supp(w)} F(S) - \kappa(S) \} \tag{if $F$ is monotone} 
 \end{align*}
 Similarly $\forall w \in \dom( f^-)$ we can write
 \begin{align*}
 f^-(w)  &= \max_{\kappa \in \R^d}  \{  \kappa^T |w| + \min_{S \subseteq V} F(S) - \kappa(S) \}  \\
 &=  \Theta_\infty(w)= \max_{\kappa \in \R^d_+}  \{  \kappa^T w + \min_{S \subseteq \supp(w)} F(S) - \kappa(S) \} \tag{if $F$ is monotone} 
 \end{align*}
\end{proposition}
\begin{proof}
$\forall w \in \dom(\Theta_\infty)$, strong duality holds by Slater's condition, hence
\begin{align*}
 \Theta_\infty(w) &= \min_{\alpha} \{ \sum_{S \subseteq V} \alpha_S F(S) :  \sum_{S \subseteq V} \alpha_S \1_S \geq |w|, \sum_{S \subseteq V} \alpha_S =1, \alpha_S \geq 0\}.\\
 &= \min_{\alpha \geq 0} \max_{\rho \in \R, \kappa \in \R^d_+}\{ \sum_{S \subseteq V} \alpha_S F(S) + \kappa^T( |w| - \sum_{S \subseteq V} \alpha_S \1_S)  +\rho(1- \sum_{S \subseteq V} \alpha_S )\}.\\
  &= \max_{\rho \in \R, \kappa \in \R^d_+} \min_{\alpha \geq 0} \{  \kappa^T |w| + \sum_{S \subseteq V} \alpha_S ( F(S) - \kappa^T \1_S - \rho )  +\rho\}.\\
  &= \max_{\rho \in \R, \kappa \in \R^d_+}  \{  \kappa^T |w| +  \rho:  F(S) \geq  \kappa^T \1_S + \rho ) \}.\\
  &= \max_{\kappa \in \R^d_+}  \{  \kappa^T |w| + \min_{S \subseteq V} F(S) - \kappa(S)  \}.
\end{align*}
Let $J = \supp(|w|)$ then $\kappa^*_{J^c} = 0$. Then for monotone functions $F(S) - \kappa^*(S) \geq F(S \cap J) - \kappa^*(S)$, so we can restrict the minimum to $S \subseteq J$.
The same proof holds for $f^-$, with the Lagrange multiplier $\kappa \in \R^d$ not constrained to be positive.
\end{proof}

The following Corollary derives the variational form \eqref{eq:SupportFormNonHomo} for $p \in [1,\infty]$.
\begin{corollary} %we can extend to non-monotone, but this is simpler.
Given any set function $F:2^V \to \R \cup \{+\infty\}$, $\Theta_p$ can be written $\forall w \in \dom(\Theta_p)$ as
\begin{align*}
\Theta_p(w) &= \max_{\kappa \in \R^d}  \sum_{j=1}^d \psi_j(\kappa_j,w_j)  +  \min_{S \subseteq V} F(S) - \kappa(S).  \\
&= \max_{\kappa \in \R_+^d}  \sum_{j=1}^d \psi_j(\kappa_j,w_j)  +  \min_{S \subseteq V} F(S) - \kappa(S). \tag{if $F$ is monotone} 
\end{align*} 
%if $\exists \supp(x) \subseteq \supp(s)$ s.t $f_\infty(s) < + \infty$, and 
where 
\begin{align*}
\psi_j(\kappa_j,w_j)  &:=\begin{cases} \kappa_j^{1/q}|w_j| &\text{ if $|w_j| \leq \kappa_j^{1/p}, \kappa_j \geq 0$}\\
\frac{1}{p}  |w_j|^{p}+ \frac{1}{q}   \kappa_j  &\text{otherwise}
\end{cases} 
\end{align*} 
\end{corollary}
\begin{proof}
By Propositions \ref{prop:LpregNonHomo} and \ref{prop:SuppFctNonHomo}, we have $\forall w \in \dom(\Theta_p)$, i.e., $\exists \eta \in [0,1]^d,$ s.t $\supp(w) \subseteq \supp(\eta), \eta \in \dom(f^-)$,
\begin{align*} %can't replace inf by min unless we define 0/0  = 0, which we'll mention in main text
\Theta_p(w) &=\inf_{\eta \in [0,1]^d} \frac{1}{p}  \sum_{j=1}^d \frac{|w_j|^{p}}{\eta_j^{p -1}} + \frac{1}{q}  f^-(\eta) \\
  &= \inf_{\eta \in [0,1]^d} \frac{1}{p}  \sum_{j=1}^d \frac{|w_j|^{p}}{\eta_j^{p -1}} + \frac{1}{q} \max_{\rho \in \R, \kappa \in \R^d}  \{  \kappa^T \eta+  \rho:  F(S) \geq  \kappa^T \1_S + \rho  \}.\\
  &\stackrel{\star}{=}  \max_{\rho \in \R, \kappa \in \R^d}  \inf_{\eta \in [0,1]^d}   \{\frac{1}{p}  \sum_{j=1}^d \frac{|w_j|^{p}}{\eta_j^{p -1}} + \frac{1}{q}   \kappa^T \eta+  \rho:  F(S) \geq  \kappa^T \1_S + \rho  \}.
\end{align*}
 $(\star)$ holds by Sion's minimax theorem \cite[Corollary 3.3]{sion1958general}. 
Note also that %for $\kappa_i \geq 0$, 
\begin{align*}
\inf_{\eta_j \in [0,1]} \frac{1}{p}  \frac{|w_j|^{p}}{\eta_j^{p -1}} + \frac{1}{q}   \kappa_j \eta_j &=\begin{cases} \kappa_j^{1/q}|w_j| &\text{ if ${|w_j|} \leq {\kappa_j^{1/p}}, \kappa_j \geq 0$}\\
\frac{1}{p}  |w_j|^{p}+ \frac{1}{q}   \kappa_j  &\text{otherwise}
\end{cases} := \psi_j(\kappa_j,w_j) 
\end{align*} 
where the minimum is $\eta^*_j = 1$ if $\kappa_j \leq 0$. If $\kappa_j \geq 0$, 
the infimum is zero if $w_j = 0$. Otherwise, the minimum is achieved at $\eta^*_j = \min\{\frac{|w_j|}{\kappa_j^{1/p}},1\}$ (if $\kappa_j = 0, \eta^*_j =1$). 
Hence,
\begin{align*}
\Theta_p(w) &= \max_{\kappa \in \R^d}  \sum_{j=1}^d \psi_j(\kappa_j,w_j)  +  \min_{S \subseteq V} F(S) - \kappa(S).
\end{align*} 
\end{proof}

 \subsection{Necessary conditions for support recovery (Proof of Theorem \ref{them:NecessaryStableGeneral})}
 
 Before proving Theorem \ref{them:NecessaryStableGeneral}, we need the following technical Lemma.
 
%\begin{restatable}{lemma}{primeDecomposableLem}
\begin{lemma}\label{lem:decomposable}
Given $J \subset V$ and a vector $w$ s.t $\supp(w) \subseteq J$, if $\Phi$ is {not decomposable at $w$ w.r.t $J$}, then $\exists  i \in J^c$ such that the $i$-th component of all subgradients at $w$ is zero; $0  = [ \partial \Phi ({w}) ]_i$.
\end{lemma}
%\end{restatable}
\begin{proof}
 If $\Phi$ is not decomposable at $w$ and $0 \not = [ \partial \Phi ({w}) ]_i, \forall i \in J^c$, then $\forall M_J >0, \exists \Delta \not = 0, \supp(\Delta) \subseteq J^c$ s.t., $\Phi(w + \Delta) <  \Phi(w) + M_J \| \Delta\|_\infty$. In particular,
 we can choose $M_J = \inf_{i \in J^c,v \in  \partial \Phi ({w}_J), v_i \not = 0 } |v_i| >0$,  if the inequality holds for some $\Delta \not = 0$, then let $i_{\max}$ denote the index where $|\Delta_{i_{\max}}|  =  \| \Delta\|_\infty$. Then given any $v \in  \partial \Phi ({w})$ s.t., $v_{i_{\max}} \not = 0$, we have
\begin{align*}
\Phi( w+ \| \Delta \|_\infty \1_{i_{\max}}) \leq \Phi( w + \Delta )  &<  \Phi(w ) + M_J \| \Delta\|_\infty \\
&\leq  \Phi(w) + \langle  v ,\| \Delta \|_\infty \1_{i_{\max}} \sign(v_{i_{\max}}) \rangle  \\
&\leq \Phi(w + \| \Delta\|_\infty \1_{i_{\max}})
\end{align*}
which leads to a contradiction.
\end{proof}

\primeNecStableProp*
\begin{proof}
We will show in particular that $\Phi$ is decomposable at $\hat{w}$ w.r.t $\supp(\hat{w})$.
Since $L$ is strongly-convex, given $z$ the corresponding minimizer $\hat{w}$ is unique, then the function $h(z) := \argmin_{w \in \R^d} L(w) - z^Tw + \lambda \Phi(w)$ is well defined.
%Given any weakly unstable $J$ and any $w$ s.t. $\supp(w) \subseteq J$, 
%We know by lemma \ref{lem:decomposable} that there exists an $i \in J^c$ such that $0 = [\partial \Phi(w)]_i$. Then
%\begin{align*}
%P(h(z) \text{ is weakly unstable}) = P(\cup_{J \text{ weakly unstable}} \cup_{i \in J^c} \supp(h(z)) \subseteq J,  [\partial \Phi(h(z))]_i= 0)
%\end{align*}
We want to show that
\begin{align*}
&P(\forall z, \text{ $\Phi$ is decomposable at $h(z)$ w.r.t $\supp(h(z))$ }) \\
&= 1 - P(\exists z, \text{s.t, $\Phi$ is not decomposable at $h(z)$ w.r.t $\supp(h(z))$ } )\\
&\geq 1 - P(\exists z, \text{ s.t., } \exists i \in \left( \supp(h(z)) \right)^c, [\partial \Phi(h(z))]_i = 0 ) &\text{by lemma \ref{lem:decomposable} }\\
&= 1.
\end{align*}

%Then for any $z$ such that $h(z)$ is weakly unstable, there must exists $J \subseteq V, i \in J^c$ such that $J$ is weakly unstable, $\supp(h(z)) \subseteq J$ and $0 = [\partial \Phi(h(z))]_i$. 
%We show that the set of $z$ such that $\Phi$ is not decomposable at $h(z)$ w.r.t $\supp(h(z))$ has measure zero. 
Given fixed $i \in V$, we show that the set $S_i := \{ z : i \in \left( \supp(h(z)) \right)^c, [\partial \Phi(h(z))]_i = 0 \}$ has measure zero. Then, taking the union of the finitely many sets $S_i, \forall i \in V$, all of zero measure, we have $P(\exists z, \text{ s.t., } \exists i \in \left( \supp(h(z)) \right)^c, [\partial \Phi(h(z))]_i = 0 ) = 0$ .

To show that the set $S_i$ has measure zero, let $z_1, z_2 \in S_i$ and denote by $\mu>0$ the strong convexity constant of $L$. We have by convexity of $\Phi$:
\begin{align*}
\Big( \big( z_1 - \nabla L(h(z_1)) \big) - \big( z_2 - \nabla L(h(z_2)) \big) \Big)^\top \Big( h(z_1)- h(z_2)\Big) &\geq 0\\
(z_1 - z_2)^\top(h(z_1)- h(z_2))  &\geq  \big( \nabla L(h(z_1))  - \nabla L(h(z_2)) \big)^\top \big( h(z_1)- h(z_2) \big) \\
(z_1 - z_2)^\top(h(z_1)- h(z_2))  &\geq  \mu \| h(z_1)- h(z_2)\|_2^2 \\
\frac{1}{\mu} \|z_1 - z_2\|_2  &\geq \| h(z_1)- h(z_2)\|_2
\end{align*}
Thus $h$ is a deterministic Lipschitz-continuous function of $z$. 
Let $J = \supp(h(z))$, then
by optimality conditions $z_J - \nabla L(h(z_J))  \in   \partial \Phi (h(z_J))$ (since $h(z) = h(z_J)$), then $z_i - \nabla L(h(z_J))_i  =0$ since $[\partial \Phi(h(z_J))]_i = 0$.
% since $[\partial \Phi(h(z))]_i = [\partial \Phi(h(z'))]_i = 0$
%If $\Phi$ is {not decomposable at $\hat{w}$ with respect to $J$}, we know by lemma \ref{lem:decomposable} that there exists an $i \in J^c$ such that $0 = [\partial \Phi(\hat{w})]_i$, this implies that $z_i - \nabla L(\hat{w})_i  =0$ 
and thus $z_i$ is a Lipschitz-continuous function of $z_J$, which can only happen with zero measure.
\end{proof}

\subsection{Sufficient conditions for support recovery (Proof of Lemma \ref{lem:Majorizer} and Theorem \ref{Thm:Consistency})}

\primeMajoLem*
\begin{proof}
The function $w \rightarrow w^\alpha$ is concave on $\R_+ \setminus \{0\}$, hence 
\begin{align*}
|w_j|^\alpha &\leq |w^0_j|^\alpha + \alpha |w^0_j|^{\alpha-1} (|w_j| - |w_j|^0) \\
|w_j|^\alpha &\leq (1 - \alpha) |w^0_j|^\alpha + \alpha |w^0_j|^{\alpha-1} |w_j| \\
%|w_j|^\alpha &\leq (1 - \alpha) |w^0_j|^\alpha + \alpha |w^0_j|^{\alpha-1} |w_j| \tag{by $\Delta$-inequality of $| \cdot |$}\\
\Phi(|w|^\alpha) &\leq \Phi((1 - \alpha) |w^0|^\alpha + \alpha |w^0|^{\alpha-1} \circ|w_j|) \tag{by monotonicity}\\
\Phi(|w|^\alpha) &\leq (1-\alpha) \Phi(  |w^0|^\alpha) + \alpha \Phi(|w^0|^{\alpha-1} \circ |w|  ) \tag{by convexity}\\
\end{align*}
If $w_j= 0$ for any $j$, the upper bound goes to infinity and hence it still holds.
\end{proof}

%Before proceeding to the proof of theorem \ref{Thm:Consistency}, we will need the following lemma.
%
%\begin{lemma}\label{lem:homogIneq}
%Given any normalized ($\Phi(0) = 0$) convex function, we have for all $w \in \R^d$:
%\begin{align*}
%t  \Phi(w) &\geq \Phi(t w) , \forall t \leq 1\\ 
%t  \Phi(w)  &\leq  \Phi(t w) , \forall t \geq 1.
%\end{align*}
%\end{lemma}
%\begin{proof}
%The first inequality holds since $\forall t \leq 1, \Phi(t w + (1 - t) 0) \leq t \Phi(w)  + (1-t) \Phi(0) =  t \Phi(w)$.
%The second inequality follows directly from the first one, since $\forall t \geq 1, \Phi(w)  =  \Phi( \frac{1}{t} ( t w) ) \leq \frac{1}{t}  \Phi(t w)$.
%\end{proof}

\primeConsistThem*
\begin{proof}
We will follow the proof in \cite{zou2006adaptive}. 
We write $\hat{w} = w^* + \frac{\hat{u}}{\sqrt{n}}$ and $\Psi_n(u)  = \frac{1}{2} \| y - X(w^* + \frac{{u}}{\sqrt{n}})\|_2^2 + \lambda_n \Phi(c \circ |w^* + \frac{{u}}{\sqrt{n}}|)$, where $c = |{w^0}|^{\alpha-1}$. Then $\hat{u} = \argmin_{u \in \R^d} \Psi_n(u)$.
Let $V_n(u) = \Psi_n(u) - \Psi_n(0)$, then 
$$ V_n(u) = \frac{1}{2} u^T Q u - \epsilon^T \frac{X u}{\sqrt{n}} + {\lambda_n}\big(  \Phi(c \circ |w^* + \frac{{u}}{\sqrt{n}}|) - \Phi(c \circ |w^*|)\big)$$

Since $w^0$ is a $\sqrt{n}$-consistent estimator to $w^*$, then $\sqrt{n} w^0_{J^c} = O_p(1)$ and $n^{\frac{1-\alpha}{2}} c^{-1}_{J^c} = O_p(1)$. Since $\frac{\lambda_n}{{n}^{\alpha/2}} \to \infty$, by stability of $J$, we have  %by Slutsky's theorem 
\begin{align}\label{eq:limAtzero}
{\lambda_n} \big( \Phi(c \circ |w^* + \frac{{u}}{\sqrt{n}}|) -  \Phi(c \circ |w^*|) \big)
&= {\lambda_n} \big( \Phi(c_{J} \circ |w_J^* + \frac{{u_J}}{\sqrt{n}}| + c_{J^c} \circ  \frac{{|u_{J^c} |}}{\sqrt{n}}) -  \Phi(c_J \circ |w_J^*|) \big) \nonumber \\ 
&\geq {\lambda_n} \big( \Phi(c_{J} \circ |w_J^* + \frac{{u_J}}{\sqrt{n}}| ) + M_J \|c_{J^c} \circ  \frac{{|u_{J^c} |}}{\sqrt{n}}\|_\infty -  \Phi(c_J \circ |w_J^*|) \big)  \nonumber\\  %\tag{by def. \ref{def:ContStability} } 
&= {\lambda_n} \big( \Phi(c_{J} \circ |w_J^* + \frac{{u_J}}{\sqrt{n}}| ) -  \Phi(c_J \circ |w_J^*|) \big)  + M_J \| {\lambda_n} n^{-\alpha/2} n^{\frac{\alpha-1}{2}}  c_{J^c} \circ {|u_{J^c} |}\|_\infty \nonumber \\ 
&\xrightarrow{p} \infty \quad \text{if $u_{J^c}\not =0$} 
%&\begin{cases} \xrightarrow{p} \infty &\text{if $u_{J^c}\not =0$} \\
%=  {\lambda_n} \big( \Phi(c_{J} \circ |w_J^* + \frac{{u_J}}{\sqrt{n}}| ) -  \Phi(c_J \circ |w_J^*|) \big) &\text{otherwise}.
%\end{cases}
\end{align}
Otherwise, if $u_{J^c} = 0$, we argue that 
\begin{equation}\label{eq:limAtNonZero}
{\lambda_n} \big( \Phi(c \circ |w^* + \frac{{u}}{\sqrt{n}}|) -  \Phi(c \circ |w^*|) \big) = 
\lambda_{n} ( \Phi(c_J \circ |w_J^* + \frac{{u_J}}{\sqrt{n}}|) - \Phi(c_J \circ |w_J^*|)) \xrightarrow{p} 0.
\end{equation}
To see this note first that since $w^0$ is a $\sqrt{n}$-consistent estimator to $w^*$, then $c_{J} = |w_{J}^0|^{\alpha-1} \xrightarrow{p} |w_{J}^*|^{\alpha-1} $,  $c_{J} \circ |w_J^*| \xrightarrow{p} |w_{J}^*|^{\alpha}$ and $c_J \circ |w_J^* + \frac{{u_J}}{\sqrt{n}}| \xrightarrow{p} |w_{J}^*|^{\alpha}$.
 Then by the assumption $|w^*|^\alpha \in \intr ~\dom~\Phi$, we have that both $c_{J} \circ |w_J^*|, c_J \circ |w_J^* + \frac{{u_J}}{\sqrt{n}}| \in \intr ~\dom~\Phi$ with probability going to one.
 By convexity, we then have:
 \begin{align*}
 \lambda_{n} ( \Phi(c_J \circ |w_J^* + \frac{{u_J}}{\sqrt{n}}|) - \Phi(c_J \circ |w_J^*|))  &\geq   \langle \nabla \Phi(c_J \circ |w_J^*|) ,  \lambda_{n} \frac{{u_J}}{\sqrt{n}} \rangle \\
 \lambda_{n} ( \Phi(c_J \circ |w_J^* + \frac{{u_J}}{\sqrt{n}}|) - \Phi(c_J \circ |w_J^*|))  &\leq   \langle \nabla \Phi( c_J \circ |w_J^* + \frac{{u_J}}{\sqrt{n}} |) ,  \lambda_{n} \frac{{u_J}}{\sqrt{n}} \rangle
\end{align*}  
where $\nabla \Phi(w)$ denotes a subgradient of $\Phi$ at $w$.

For all $w \in  \intr~\dom~\Phi$ where $\Phi$ is convex, monotone and normalized, we have that $\| z \|_\infty < \infty, \forall z \in \partial \Phi(w)$. 
To see this, note that since $w \in  \intr ~\dom ~\Phi$, $\exists \delta>0$ s.t., $\forall w' \in B_{\delta}(w), \Phi(w') < +\infty$. Let $w' = w + \sign(z) \1_{i_{\max}} \delta$, where $i_{\max}$ denotes the index where $|z_{i_{\max}}| = \| z \|_\infty$ then by convexity we have
\begin{align*}
\Phi(w') &\geq \Phi(w) + \langle z, w' - w\rangle, &\forall z \in \partial \Phi(w) \\
+ \infty > \Phi(w') &\geq   \| z \|_\infty \delta,  &\forall z \in \partial \Phi(w),  &\quad \text{(since $\Phi(w) \geq 0$)}
\end{align*}

%Since $\Phi$ is convex and normalized, then by lemma \ref{lem:homogIneq}, it follows that:
%{
%\begin{align}\label{eq:UppBoundAtNonZero}
%  {\lambda_n} \big( \Phi(c_{J} \circ |w_J^* + \frac{{u_J}}{\sqrt{n}}| ) -  \Phi(c_J \circ |w_J^*|) \big)  &\leq \frac{\lambda_n}{2} \big( \Phi(2 c_{J} \circ |w_J^*|) +\Phi(2 c_{J} \circ \frac{|{u_J}|}{\sqrt{n}} ) \big) -  {\lambda_n} \Phi(c_J \circ |w_J^*|)  \\
%  &\leq   \Phi(c_J \circ \frac{ \lambda_{n}|u_J|}{\sqrt{n}}) \tag{if $\lambda_n \geq 2$}
%\end{align}}
Since $\frac{\lambda_n}{\sqrt{n}} \to 0$, we can then conclude by Slutsky's theorem that \eqref{eq:limAtNonZero} holds.

Hence by \eqref{eq:limAtzero} and \eqref{eq:limAtNonZero},
\begin{align}\label{eq:limPhi}
{\lambda_n}\big(  \Phi(c \circ |w^* + \frac{{u}}{\sqrt{n}}|) - \Phi(c \circ |w^*|)\big) &\xrightarrow{p} \begin{cases}
0 & \text{if $u_{J^c}=0$}\\
\infty &\text{Otherwise} 
\end{cases}.
\end{align}

By CLT, $\frac{X^T\epsilon}{\sqrt{n}}   \xrightarrow{d} W \sim \mathcal{N}(0,\sigma^2 Q)$, it follows then that $ V_n(u) \xrightarrow{d} V(u)$, where
\begin{align*}
V(u) &= \begin{cases}
 \frac{1}{2} u_{J}^T Q_{J J} u_{J} - W^T_{J} u_{J} & \text{if $u_{{J}^c}=0$}\\
\infty &\text{Otherwise} 
\end{cases}.
\end{align*}

$V_n$ is convex and the unique minimum of $V$ is $u_{J} = Q^{-1}_{J J} W_{J}, u_{{J}^c} = 0$, hence by epi-convergence results [c.f., \cite{zou2006adaptive}]
\begin{align}\label{eq:AsympNormal}
\hat{u}_{J} \xrightarrow{d} Q^{-1}_{J J} W_{J} \sim \mathcal{N}(0,\sigma^2 Q^{-1}_{J J} ), \quad \hat{u}_{{J}^c} \xrightarrow{d} 0.
\end{align}
Since $\hat{u} = \sqrt{n}(\hat{w} - w^*)$, then it follows from \eqref{eq:AsympNormal} that
\begin{align}\label{eq:convergProb}
\hat{w}_{J} \xrightarrow{p}  w^*_{J}, &\quad \hat{w}_{{J}^c} \xrightarrow{p} 0 
\end{align}
Hence, $P(\supp(\hat{w}) \supseteq J) \to 1$ and it is sufficient to show that $P(\supp(\hat{w}) \subseteq  J) \to 1$ to complete the proof.\\

For that denote $\hat{J} = \supp(\hat{w})$ and let's consider the event $\hat{J} \setminus J \not = \emptyset$.
By optimality conditions, we know that
\begin{align*}
- X_{\hat{J} \setminus J}^T(X\hat{w} - y )  \in \lambda_n  [\partial \Phi(c \circ \cdot)(\hat{w})]_{\hat{J} \setminus J} 
\end{align*}
Note, that $- \frac{X_{\hat{J} \setminus J}^T(X\hat{w} - y )}{\sqrt{n}} = \frac{X_{\hat{J} \setminus J}^TX(\hat{w} - w^* )}{\sqrt{n}} - \frac{X_{\hat{J} \setminus J}^T\epsilon }{\sqrt{n}} $.
By CLT, $\frac{X_{\hat{J} \setminus J}^T\epsilon }{\sqrt{n}}   \xrightarrow{d} W \sim \mathcal{N}(0,\sigma^2 Q_{{\hat{J} \setminus J} ,{\hat{J} \setminus J}})$ and by \eqref{eq:convergProb} $\hat{w} - w^* \xrightarrow{p} 0$ then $- \frac{X_{\hat{J} \setminus J}^T(X\hat{w} - y )}{\sqrt{n}}  = O_p(1)$.\\%by proposition 91 in myNotes

On the other hand, $\frac{\lambda_n c_{\hat{J} \setminus J} }{\sqrt{n}}  = \lambda_n n^{\frac{1-\alpha}{2}} n^{\frac{\alpha-1}{2}} c_{\hat{J} \setminus J} \to \infty$, hence $\frac{\lambda_n c_{\hat{J} \setminus J} }{\sqrt{n}}  c^{-1}_{\hat{J} \setminus J}v_{\hat{J} \setminus J}\to \infty$, $\forall v \in  \partial \Phi(c \circ \cdot)(\hat{w})$, since $c^{-1}_{\hat{J} \setminus J} v_{\hat{J} \setminus J} = O_p(1)^{-1}$. 
To see this, let $w'_J = \hat{w}_J$ and $0$ elsewhere. Note that by definition of the subdifferential and the stability assumption on $J$, there must exists $M_J>0$ s.t  %choosing w' = what - what_Jhat-J
\begin{align*}
\Phi(c \circ w') &\geq \Phi(c \circ \hat{w} ) + \langle v_{\hat{J} \setminus J}  ,  -\hat{w}_{\hat{J} \setminus J} \rangle\\
 &\geq \Phi(c \circ w') + M_J \|c_{\hat{J} \setminus J} \circ \hat{w}_{\hat{J} \setminus J} \|_\infty  - \|c^{-1}_{\hat{J} \setminus J} \circ v_{\hat{J} \setminus J} \|_1  \|c_{\hat{J} \setminus J} \circ\hat{w}_{\hat{J} \setminus J}\|_\infty\\
\|c^{-1}_{\hat{J} \setminus J} \circ v_{\hat{J} \setminus J} \|_1 &\geq M_J
\end{align*}
 We deduce then $P(\supp(\hat{w}) \subseteq  J) = 1 - P(\hat{J} \setminus J \not = \emptyset) = 1 - P(\text{optimality condition holds}) \to 1$.
\end{proof}

\subsection{Discrete stability (Proof of Proposition \ref{prop:EqDefStableProp} and relation to weak submodularity)}

\primeEqDefStableProp*
\begin{proof}
If $F$ is $\rho$-submodular and $J$ is weakly stable, then $\forall A \subseteq J,  \forall i \in J^c, 0 <\rho [ F(J \cup \{i\}) - F(J) ] \leq F(J \cup \{i\}) - F(J) $, i.e., $J$ is strongly stable w.r.t. $F$.
If $F$ is such that any weakly stable set is also strongly stable, then if $F$ is not $\rho$-submodular, then $\forall \rho \in (0,1]$ there must exists a set $B \subseteq V$, s.t., $\exists A \subseteq B, i \in B^c$, s.t., $ \rho [ F(B \cup \{i\}) - F(B) ] > F(A \cup \{i\}) - F(A) \geq 0$. Hence, $F(B \cup \{i\}) - F(B)>0$, i.e., $B$ is weakly stable and thus it is also strongly stable and we must have $F(A \cup \{i\}) - F(A)>0$. Choosing then in particular, $\rho = \min_{B \subseteq V} \min_{A \subseteq B, i \in B^c} \frac{F(A \cup \{i\}) - F(A)}{F(B \cup \{i\}) - F(B)} \in (0,1]$, leads to a contradiction; $\min_{A \subseteq B, i \in B^c} {F(A \cup \{i\}) - F(A)} \geq \rho [ F(B \cup \{i\}) - F(B) ] > F(A \cup \{i\}) - F(A) $.
\end{proof}

We show that $\rho$-submodularity is a stronger condition than weak submodularity. First, we recall the definition of weak submodular functions.

\begin{definition}[Weak Submodularity (c.f., \cite{das2011submodular, elenberg2016restricted})]
A function $F$ is weakly submodular if $\forall S, L, S \cap L = \emptyset, F(L \cup S) - F(L)>0$, 
$$ \gamma_{S,L} = \frac{\sum_{i \in S} F(L \cup \{i\}) - F(L)}{ F(L \cup S ) - F(L)} >0$$ 
\end{definition}

\begin{proposition}
If $F$ is $\rho$-submodular then $F$ is weakly submodular. But the converse is not true.
\end{proposition}
\begin{proof}
If $F$ is $\rho$-submodular then $\forall S, L, S \cap L = \emptyset, F(L \cup S) - F(L)>0$, let $S = \{i_1, i_2, \cdots, i_r\}$
\begin{align*}
F(L \cup S) - F(L) &= \sum_{k =1}^r F(L \cup \{i_1, \cdots, i_k\}) - F(L \cup  \{i_1, \cdots, i_{k-1}\})\\
&\leq  \sum_{k =1}^r \frac{1}{\rho} (F(L \cup \{ i_k\}) - F(L ) ) \\
\Rightarrow  \gamma_{S,T} &= \rho >0.
\end{align*}

We show that the converse is not true by giving a counter-example: Consider the function defined on $V=\{1,2,3\}$, where $F(\{i\}) = 1, \forall i, F(\{1,2\})=1,  F(\{2,3\})=2,  F(\{1,3\})=2,  F(\{1,2,3\})=3$. Then note that this function is weakly submodular. We only need to consider sets $|S|\geq 2$, since otherwise $\gamma_{S,T}>0$ holds trivially. Accordingly, we also only need to consider $L$ which is the empty set or a singleton. In both cases $\gamma_{S,T}>0$. However, this $F$ is not $\rho$-submodular, since $F(1,2) - F(1) = 0 < \rho (F(1,2,3) - F(1,3)) = \rho$ for any $\rho>0$.
\end{proof}

\subsection{Relation between discrete and continuous stability (Proof of Propositions \ref{prop:DS-CS} and \ref{prop:CS-DSpinf}, and Corollary \ref{corr:SubStEq})}

First, we present a useful simple lemma, which provides an equivalent definition of decomposability for monotone function.

\begin{lemma}\label{lem:eqDecompDef}
Given $w \in \R^d, J \subseteq J, \supp(w) = J$, if $\Phi$ is a monotone function, then $\Phi$ is decomposable at $w$ w.r.t $J$ iff  $\exists M_J>0, \forall \delta>0, i \in J^c,$ s.t, 
$$\Phi(w + \delta \1_i) \geq  \Phi(w) + M_J \delta.$$
\end{lemma}
\begin{proof}
By definition \ref{def:ContStability}, $ \exists M_J>0, \forall \Delta \in \R^d, \supp(\Delta) \subseteq J^c,$
$$\Phi(w + \Delta) \geq  \Phi(w) + M_J \| \Delta\|_\infty.$$
in particular this must hold for $\Delta = \delta \1_i$. On the other hand, if the inequality hold for all $\delta \1_i$, then given any $\Delta$ s.t. $\supp(\Delta) \subseteq J^c$ let $i_{\max}$ be the index where $\Delta_{i_{\max} }= \| \Delta \|_\infty$ and let $\delta = \| \Delta \|_\infty$, then 
\begin{align*}
\Phi(w + \Delta)  \geq \Phi(w + \delta _{i_{\max} }) \geq  \Phi(w) + M_J \delta = \Phi(w) + M_J \| \Delta\|_\infty.
\end{align*} 
\end{proof}

\primeDSCSProp*
\begin{proof}
We make use of the variational form \eqref{eq:SupportForm}.
Given a set $J$ stable w.r.t to $F$ and $\supp(w) \subseteq J$, let $\kappa^* \in \argmax_{\kappa \in \R^d_+} \{ \sum_{i \in J} \kappa_i^{1/q} |w_i| : \kappa(A) \leq F(A), \forall A \subseteq V\}$, then $\Omega(w) = |w_J|^T(\kappa_J^*)^{1/q}$. 
Note that $\forall A \subseteq J, F(A \cup i) > F(A)$, by definition \ref{def:DisStable}. %$F(A \cup i) - F(A) \geq F(J \cup i) - F(J)>0$. 
Hence,
$\forall i \in J^c$, we can define $\kappa' \in \R^d_+$ s.t., $\kappa'_J = \kappa_J^*$, $\kappa'_{(J \cup i)^c} = 0$ and $\kappa'_i = \min_{A \subseteq J} F(A \cup i) - F(A)>0$.
Note that $\kappa'$ is feasible, since $\forall A \subseteq J, \kappa'(A) = \kappa^*(A) \leq F(A)$ and $\kappa'(A +  i) = \kappa^*(A) +\kappa'_i \leq F(A) +  F(A \cup i) - F(A) = F(A \cup i)$. For any other set $\kappa'(A) = \kappa'( A \cap (J+i)) \leq  F(A \cap (J+i)) \leq F(A)$, by monotonicity.
It follows then that $\Omega(w + \delta \1_i ) =  \max_{\kappa \in \R^d_+} \{ \sum_{i \in J \cup i}^d \kappa_i^{1/q} |w_i| : \kappa(A) \leq F(A), \forall A \subseteq V\} \geq |w_J|^T(\kappa_J^*)^{1/q} + \delta (\kappa'_i)^{1/q} \geq  \Omega(w) + \delta M$, with $M = (\kappa'_i)^{1/q} >0$. The proposition then follows by lemma \ref{lem:eqDecompDef}.\\

Similarly, the proof for $\Theta_p$ follows in a similar fashion.
We make use of the variational form \eqref{eq:SupportFormNonHomo}.
Given a set $J$ stable w.r.t to $F$ and $\supp(w) \subseteq J$, first note that this implicity implies that $F(J)< +\infty$ and hence $\Theta_p(w) < +\infty$. %since F(J) <F(J + e) => F(J) < + \infty, unless J = V but in this case the lemma holds trivially. If F(J) < infty and w \in [0,1], then we can cover w with J.
Let $\kappa^* \in \argmax_{\kappa \in \R_+^d}  \sum_{j=1}^d \psi_j(\kappa_j,w_j)  +  \min_{S \subseteq V} F(S) - \kappa(S)$ and $S^* \in \argmin_{S \subseteq J} F(S) - \kappa^*(S) $. 
Note that $\forall S \subseteq J, \forall i \in J^c, F(S \cup i) > F(S)$, by definition \ref{def:DisStable}. Hence,
$\forall i \in J^c$, we can define $\kappa' \in \R^d_+$ s.t., $\kappa'_J = \kappa_J^*$, $\kappa'_{(J \cup i)^c} = 0$ and $\kappa'_i = \min_{S \subseteq J} F(S \cup i) - F(S)>0$.
Note that $\forall S \subseteq J, F(S) - \kappa'(S) = F(S) -\kappa^*(S) \geq F(S^*) -\kappa^*(S^*)$ and $F(S + i) - \kappa'(S+i) = F(S + i) - \kappa^*(S) -\kappa'_i \geq  F(S +i) - \kappa^*(S) - F(S + i) + F(S) \geq  F(S^*) -\kappa^*(S^*)$.  Note also that $\psi_i(\kappa'_i,\delta) = (\kappa'_i)^{1/q} \delta$ if $\delta \leq (\kappa'_i)^{1/p}$, and $\psi_i(\kappa'_i,\delta) = \frac{1}{p} \delta^p + \frac{1}{q} \kappa'_i = \delta ( \frac{1}{p} \delta^{p-1} + \frac{1}{q} \kappa'_i \delta^{-1}) \geq \delta (\kappa'_i)^{1/q} $ otherwise.
It follows then that $\Theta_p(w + \delta \1_i )  \geq \sum_{j \in J} \psi_j(\kappa_j,w_j)  + (\kappa_i')^{1/q} \delta+  \min_{S \subseteq J \cup i} F(S) - \kappa'(S) \geq \sum_{j \in J} \psi_j(\kappa_j,w_j)  + (\kappa_i')^{1/q} \delta+ \min_{S \subseteq J} F(S) - \kappa^*(S) = \Theta_p(w) + \delta M$ with $M = (\kappa_i')^{1/q}>0$. The proposition then follows by lemma \ref{lem:eqDecompDef}.
\end{proof}

\primeCSDSpinfProp*
\begin{proof}
$
F(A + i ) = \Omega_\infty(\1_A + \1_i) = \Theta_\infty(\1_A + \1_i)  > \Omega_\infty(\1_A) = \Theta_\infty(\1_A ) = F(A), \forall A \subseteq J.$
\end{proof}

\primeSubStEqCorr*
\begin{proof}
If $F$ is a monotone submodular function, then $\Omega_\infty(w) = \Theta_\infty(w) = f_L(|w|)$. If $J$ is not weakly stable w.r.t $F$, then $\exists i \in J^c$ s.t., $F(J \cup \{i\}) = F(J)$. Thus, given any $w, \supp(w) = J$, choosing $0 < \delta < \min_{i \in J} |w_i| $, result in $f_L(|w| + \delta \1_i) = f_L(|w| )$, which contradicts the weak stability of $J$ w.r.t to $\Omega_\infty = \Theta_\infty$.
\end{proof}